\documentclass[10pt,journal,compsoc,twoside]{IEEEtran}
\ifCLASSOPTIONcompsoc
  \usepackage[nocompress]{cite}
\else
  \usepackage{cite}
\fi
\ifCLASSINFOpdf
  \usepackage[pdftex]{graphicx}
  \graphicspath{{figures},{photos}}
\else
\fi
\usepackage{amsmath}
\usepackage{amsfonts}
\usepackage{amsthm}

\usepackage{algorithm}
\usepackage{algorithmicx}
\usepackage{algpseudocode}

\ifCLASSOPTIONcompsoc
 \usepackage[caption=false,font=footnotesize,labelfont=sf,textfont=sf]{subfig}
\else
 \usepackage[caption=false,font=footnotesize]{subfig}
\fi

\usepackage[capitalize]{cleveref}
\crefname{ineq}{Ineq.}{Ineqs.}
\creflabelformat{ineq}{#2{\upshape(#1)}#3}
\usepackage{booktabs}
\usepackage{tabularx}
\usepackage{multirow}
\usepackage[T1]{fontenc}
\usepackage{xcolor}
\usepackage{xfrac}
\usepackage{bm}

\newcommand{\citet}[1]{\cite{#1}}

\usepackage{listofitems}
\usepackage{mathtools}

\newcommand\msmfmt{00.00}
\newcommand\mssep{\ensuremath{_\pm}}
\newcommand\mssfmt{_{00.00}}
\newcommand\msstdalign{l}

\newcommand{\basicmeanstyle}[1]{#1}%
\newcommand{\basicstdstyle}[1]{#1}%
\newsavebox\CBox 
\newcommand*\bfmeanstyle[1]{\sbox\CBox{\ensuremath{\basicmeanstyle{#1}}}\resizebox{\wd\CBox}{\ht\CBox}{\ensuremath{\mathbf{\basicmeanstyle{#1}}}}}%
\newcommand*\bfstdstyle[1]{\sbox\CBox{\ensuremath{_{\basicstdstyle{#1}}}}\resizebox{\wd\CBox}{\ht\CBox}{\ensuremath{\mathbf{_{\basicstdstyle{#1}}}}}}%

\newlength{\meanlen}
\newlength{\stdlen}
\newcommand{\meanstyle}[1]{\basicstdstyle{#1}}%
\newcommand{\stdstyle}[1]{\basicstdstyle{#1}}%

\newcommand{\msprintmean}[1]{\ensuremath{\mathmakebox[\meanlen][r]{\meanstyle{#1}}}}
\newcommand{\msprintstd}[1]{\mssep\ensuremath{\mathmakebox[\stdlen][\msstdalign]{_{\stdstyle{#1}}}}}

\newcommand{\setcellformat}[4][l]{
    \renewcommand\msstdalign{#1}%
    \renewcommand\msmfmt{#2}%
    \renewcommand\mssep{#3}%
    \renewcommand\mssfmt{_{#4}}%
    \settowidth{\meanlen}{\ensuremath{\msmfmt}}%
    \settowidth{\stdlen}{\ensuremath{\mssfmt}}%
}
\newcommand{\msc}[2][n]{%
    \let\meanstyle\basicmeanstyle%
    \let\stdstyle\basicstdstyle%
    \ifx b#1\let\meanstyle\bfmeanstyle\let\stdstyle\bfstdstyle\fi%
	\setsepchar{,}%
    \reademptyitems%
	\readlist*\theparam{#2}%
    \ignoreemptyitems%
	\readlist*\theneparam{#2}%
    \msprintmean{\theparam[1]}%
    \ifnum \theparamlen=2 
        \ifnum \theneparamlen=2
            \msprintstd{\theparam[2]}%
        \else
            \phantom{\msprintstd{0}}%
        \fi
    \fi
}

\newcommand{\msb}[1]{\msc[b]{#1}}

\newtheorem{theorem}{Theorem}

\newtheorem{definition}[theorem]{Definition}
\newtheorem{proposition}[theorem]{Proposition}
\newtheorem{corollary}[theorem]{Corollary}

\DeclareMathOperator{\auc}{AUC}
\DeclareMathOperator{\rpaucmo}{rpAUC}
\DeclareMathOperator{\tpr}{TPR}
\DeclareMathOperator{\fpr}{FPR}
\DeclareMathOperator*{\argmin}{argmin}
\newcommand{\samult}{\textsc{Samult}}

\newcommand{\samultpu}{\ensuremath{\textsc{Samult}^\textsc{P+U}}}
\newcommand{\rpauc}{{WSAUC}}
\newcommand{\aucbarrier}{AUC-B}
\newcommand{\aucramp}{AUC-R}
\newcommand{\auchinge}{AUC-H}
\newcommand{\aucunhinged}{AUC-U}

\DeclareMathOperator*{\Ebb}{\mathbb{E}}
\newcommand{\Ibb}{\mathbb{I}}

\newcommand{\elloi}{\ell_{01}}

\newcommand{\simiid}{\stackrel{\mathrm{i.i.d.}}{\sim}}

\newcommand{\Rcal}{\mathcal{R}}
\newcommand{\Xcal}{\mathcal{X}}

\newcommand{\Xcalp}{\mathcal{X}_{P}}
\newcommand{\Xcaln}{\mathcal{X}_{N}}
\newcommand{\Xcalu}{\mathcal{X}_{U}}
\newcommand{\Xcall}{\mathcal{X}_{L}}
\newcommand{\Xcala}{\mathcal{X}_{A}}
\newcommand{\Xcalb}{\mathcal{X}_{B}}

\newcommand{\Xcalpt}{\mathcal{X}_{\tilde{P}}}
\newcommand{\Xcalnt}{\mathcal{X}_{\tilde{N}}}

\newcommand{\risk}[1]{R_{{#1}}}
\newcommand{\emprisk}[1]{\hat{R}_{{#1}}}

\newcommand{\Rpn}{\risk{PN}}
\newcommand{\empRpn}{\emprisk{PN}}
\newcommand{\Rpu}{\risk{PU}}
\newcommand{\empRpu}{\emprisk{PU}}
\newcommand{\Run}{\risk{UN}}

\newcommand{\Rpnu}{\risk{PNU}}
\newcommand{\empRpnu}{\emprisk{PNU}}
\newcommand{\Rpp}{\risk{PP}}
\newcommand{\Rnn}{\risk{NN}}
\newcommand{\Rab}{\risk{AB}}
\newcommand{\empRab}{\emprisk{AB}}

\newcommand{\xv}{\boldsymbol{x}}

\newcommand{\xp}{\boldsymbol{x}}

\newcommand{\xn}{\boldsymbol{x}'}

\newcommand{\xpi}{\boldsymbol{x}_i}
\newcommand{\xnj}{\boldsymbol{x}'_j}
\newcommand{\xuk}{\boldsymbol{x}''_k}

\newcommand{\EP}{\Ebb_{\xp \sim p_P (\xv)}}

\newcommand{\EN}{\Ebb_{\xn \sim p_N (\xv)}}

\newcommand{\mnist}{{MNIST}}
\newcommand{\fmnist}{{FashionMNIST}}
\newcommand{\cifar}{{CIFAR10}}
\newcommand{\cifarb}{{CIFAR100}}

\begin{document}
\title{Weakly Supervised {AUC} Optimization: \\A Unified Partial {AUC} Approach}

\author{Zheng~Xie,
        Yu~Liu,
        Hao-Yuan~He,
        Ming~Li,~\IEEEmembership{Member,~IEEE},
        and~Zhi-Hua~Zhou,~\IEEEmembership{Fellow,~IEEE}%
\IEEEcompsocitemizethanks{
% \IEEEcompsocthanksitem Z. Xie, Y. Liu, H.-Y. He, M. Li, and Z.-H. Zhou are with Key Laboratory for Novel Software Technology, Nanjing University, Nanjing 210023, China.\protect\\
\IEEEcompsocthanksitem Zheng Xie, Yu Liu, Hao-Yuan He, Ming Li, and Zhi-Hua Zhou are with Key Laboratory for Novel Software Technology, Nanjing University, Nanjing 210023, China.\protect\\
Email: \{xiez, liuy, hehy, lim, zhouzh\}@lamda.nju.edu.cn%
\IEEEcompsocthanksitem Corresponding author: Ming Li.
}% <-this % stops an unwanted space
\thanks{Manuscript received Apr 12, 2023; revised Oct 15, 2023.}%
}


\markboth{{IEEE} Transactions on Pattern Analysis and Machine Intelligence,~Vol.~xx, No.~x, Month~2024}%
{Xie et al.: Weakly Supervised {AUC} Optimization: A Unified Partial {AUC} Approach}



\IEEEtitleabstractindextext{%
\begin{abstract}
Since acquiring perfect supervision is usually difficult, real-world machine learning tasks often confront inaccurate, incomplete, or inexact supervision, collectively referred to as weak supervision. 
In this work, we present WSAUC, a unified framework for weakly supervised AUC optimization problems, which covers noisy label learning, positive-unlabeled learning, multi-instance learning, and semi-supervised learning scenarios.
Within the WSAUC framework, we first frame the AUC optimization problems in various weakly supervised scenarios as a common formulation of minimizing the AUC risk on contaminated sets, and demonstrate that the empirical risk minimization problems are consistent with the true AUC{}.
Then, we introduce a new type of partial AUC, specifically, the reversed partial AUC (rpAUC), which serves as a robust training objective for AUC maximization in the presence of contaminated labels.
WSAUC offers a universal solution for AUC optimization in various weakly supervised scenarios by maximizing the empirical rpAUC{}.
Theoretical and experimental results under multiple settings support the effectiveness of WSAUC on a range of weakly supervised AUC optimization tasks.
\end{abstract}

\begin{IEEEkeywords}
AUC optimization, weakly supervised learning, partial AUC{}.
\end{IEEEkeywords}}

\maketitle

\IEEEdisplaynontitleabstractindextext

% \fi
\IEEEpeerreviewmaketitle

% \IEEEraisesectionheading{\section{Introduction}\label{sec:introduction}}

% for IEEE Computer Society journal papers produced under \LaTeX\ using
% IEEEtran.cls version 1.8b and later.

%%%%%%%%%%%%%%%%%%%%%%%%%%%%%%%%%%%%%%%%%%%%%%%%%%%%%%%%%%%%%%%%%%%%%%%%%%%%
\IEEEraisesectionheading{\section{Introduction}\label{sec:introduction}}
% \section{Introduction}

\IEEEPARstart{T}{he} ROC curve (receiver operating characteristic curve) depicts the performance of a classifier by illustrating the variation of its true positive rate (TPR) against its false positive rate (FPR) for different thresholds. Its widespread use extends to fields such as military, biology, medicine, and others. Within machine learning, the area under ROC curve~\cite{auc,a1,a2,aucsurvey}, commonly referred to as AUC, is frequently utilized as a measure of the model's classification ability, without the explicit setting of a threshold. Practically, AUC is robust to class imbalance problem and effectively depicts the ranking ability of the model. Its utility lead to AUC optimization~\cite{cortes2003auc,aucsurvey} becoming a prevalent solution for building models with high AUC performance. By optimizing AUC, the models can better deal with imbalanced data and achieve improved ranking performance. Research on AUC optimization covers topics ranging from effective batch optimization~\cite{Freund2003,joachims2005} to online optimization~\cite{aucgd,Zhao2011oam,solam}, theoretical properties~\cite{consistency,aucgeb}, and combination with deep learning~\cite{Liu2020,Yuan2022}, etc. 

Recently, studies on partial AUC (pAUC) have emerged, as researchers have argued that for certain tasks, only the TPR or FPR within a specific range is of interest~\cite{Narasimhan2013opauc}. Two major variations, One-way Partial AUC (OPAUC)~\cite{Narasimhan2013opauc} and Two-way Partial AUC (TPAUC)~\cite{Yang2019tpauc}, have been proposed, followed by studies on efficient optimization of such metrics~\cite{Cheng2019,Iwata2020sslpauc,Yang2021tpaucpie}. It has been demonstrated that the partial AUC is useful in many real-world tasks including gene detection~\cite{Liu2010bio}, speech processing~\cite{bai2020icassp}, etc.

Despite the success of AUC optimization, most of the studies in this area focus on learning from clean data. However, in real-world learning tasks, collecting enough clean data is often difficult, and it is therefore necessary to study AUC optimization in the context of weakly supervised learning~(WSL)~\cite{wsl}. Common weakly supervised learning paradigms include semi-supervised learning~\cite{zhubook}, positive-unlabeled learning~\cite{pulsurvey}, noisy label learning~\cite{Han2021nllsurvey}, multi-instance learning~\cite{milsurvey}, and more. According to \citet{wsl}, these paradigms can be categorized as learning from incomplete, inaccurate, or inexact supervision.

Due to the pairwise formulation of the AUC risk, accuracy-oriented weakly supervised learning approaches cannot be easily adapted for AUC maximization. Thus, AUC optimization under weak supervision demands dedicated research. Currently, studies in this area remain relatively limited, such as in the contexts of semi-supervised learning~\cite{Fujino2016,puauc, Xie2018} and noisy label learning~\cite{Charoenphakdee2019}. While these studies focus on specific scenarios within WSL, no research to date offers a general understanding of the series of problems in weakly supervised AUC optimization.

In this work, we present WSAUC, a unified framework for weakly supervised AUC optimization problems. The framework provides a unified view of multiple WSL scenarios including (1) noisy label learning, (2) positive-unlabeled learning, (3) multi-instance learning, and (4) semi-supervised learning with or without label noise. 
Within the WSAUC framework, we first frame the different types of weakly supervised AUC optimization problems into a unified formulation, which is in the form of AUC optimization with contaminated labels. 
By proper construct the weakly supervised AUC risks under various scenarios, the empirical risk minimization (ERM) problems are proven to be consistent with the true AUC risk, so that the problems can be solved with a unified solution.

To mitigate the effect of contamination in the unified formulation of WSAUC, we propose a new type of partial AUC, called reversed partial AUC (rpAUC, c.f.~\cref{fig:rpauc}). We show the connection between minimizing empirical rpAUC risk (or rpAUC maximization) and commonly used robust training approaches in noisy label learning, which suggests that minimizing empirical rpAUC risk leads to robust AUC optimization with contaminated labels.
This method can be easily combined with the ERM problems of AUC optimization, and can be simply implemented by slightly modifying existing pAUC maximization algorithms.
Based on rpAUC, the WSAUC framework offers a unified approach for AUC optimization in various weakly supervised scenarios.
We evaluate the WSAUC framework through experiments under different WSL settings, and conclude its effectiveness.

The rest of this paper is organized as follows: We first discuss the related work in \cref{sec:related_work}, and then introduce preliminary knowledge in \cref{sec:preliminary}. In \cref{sec:unified}, we present the unified view of weakly supervised AUC optimization, which is followed by the theoretical analysis in \cref{sec:theory}. In \cref{sec:robust}, we introduce the novel concept of rpAUC for robust AUC optimization. Experimental results are presented in \cref{sec:exp}. Finally, \cref{sec:con} concludes with future work.

\begin{figure}[t]
	\centering
	\captionsetup[subfigure]{margin=0pt} 
	% \hspace{-0.04\textwidth}
  % adding % after the subfloat line can further decrease the space
	\subfloat[AUC]{
		\label{fig:auc}
		\includegraphics[width=0.22\textwidth,trim=6.5cm 1.4cm 6.5cm 1cm]{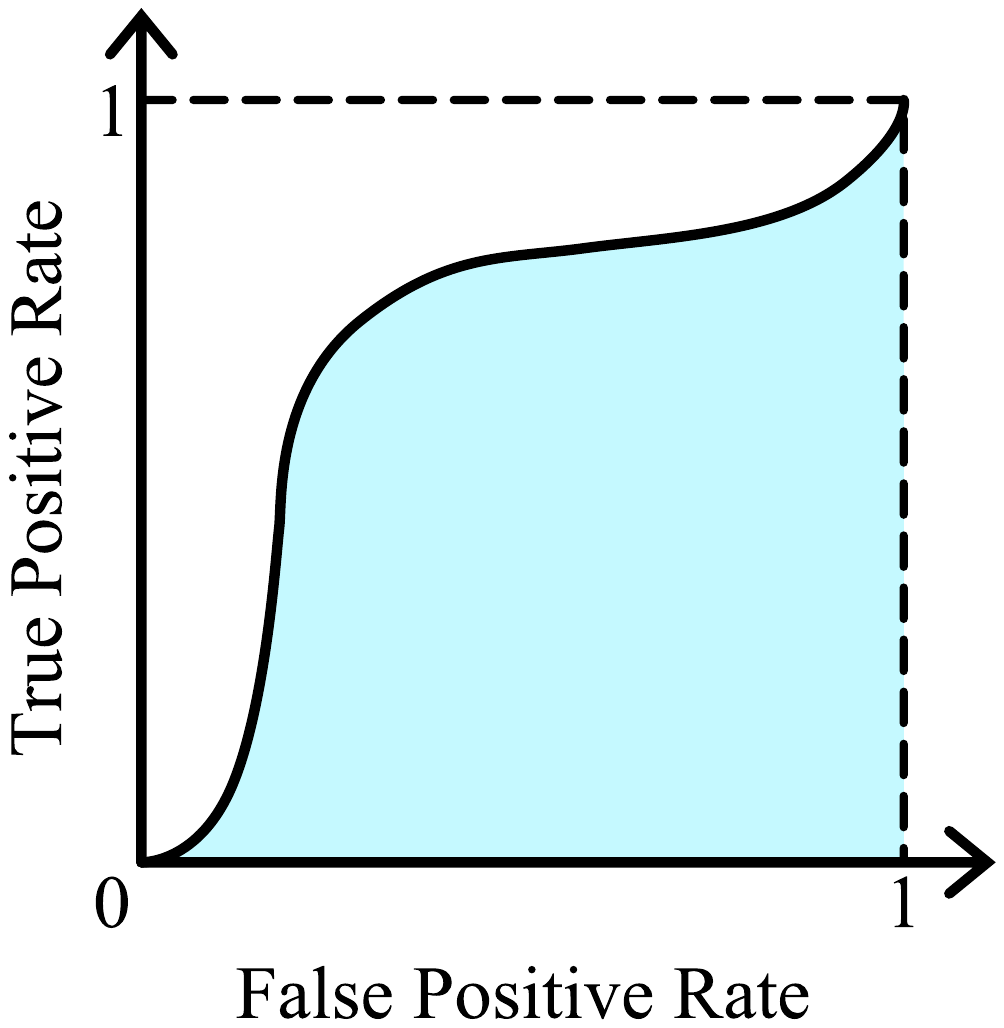}
  }
  \subfloat[OPAUC]{
		\label{fig:opauc}
		\includegraphics[width=0.22\textwidth,trim=6.5cm 1.4cm 6.5cm 1cm]{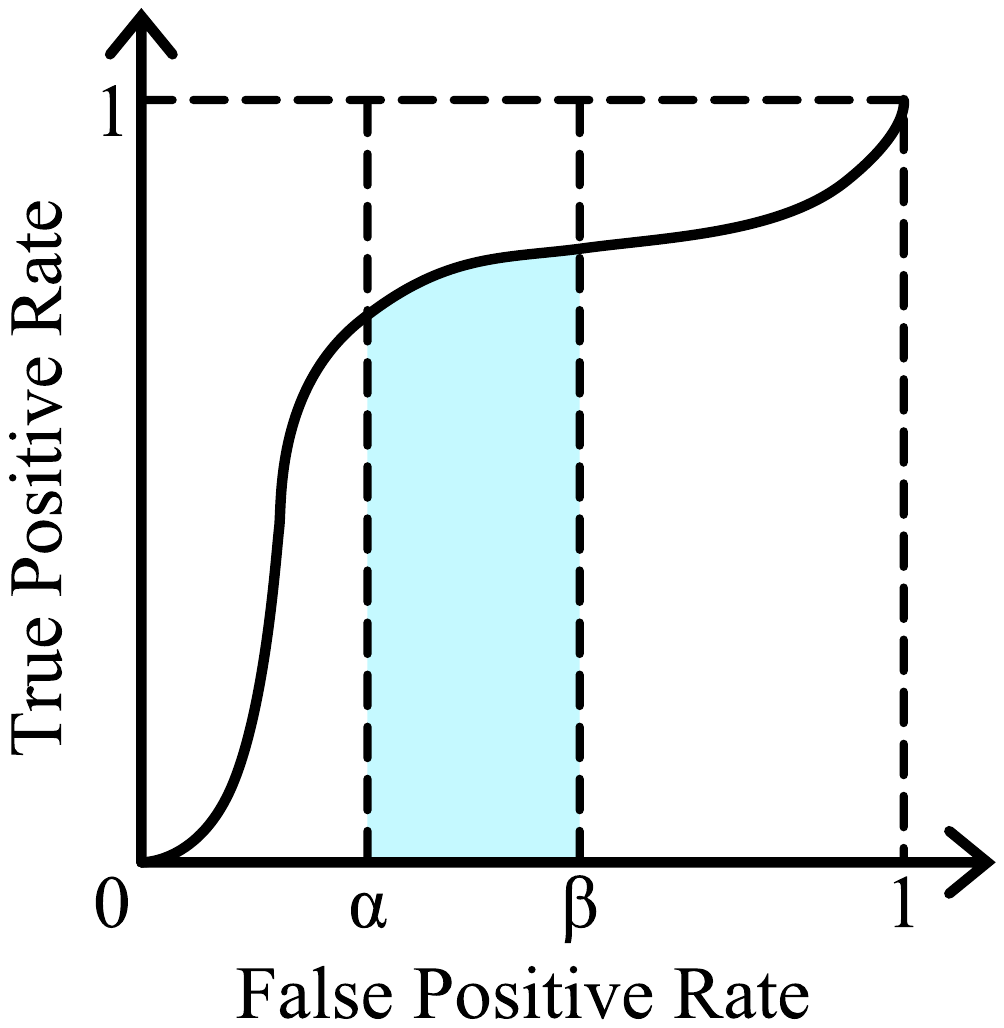}
	}\\\subfloat[TPAUC]{
		\label{fig:tpauc}
		\includegraphics[width=0.22\textwidth,trim=6.5cm 1.4cm 6.5cm 1cm]{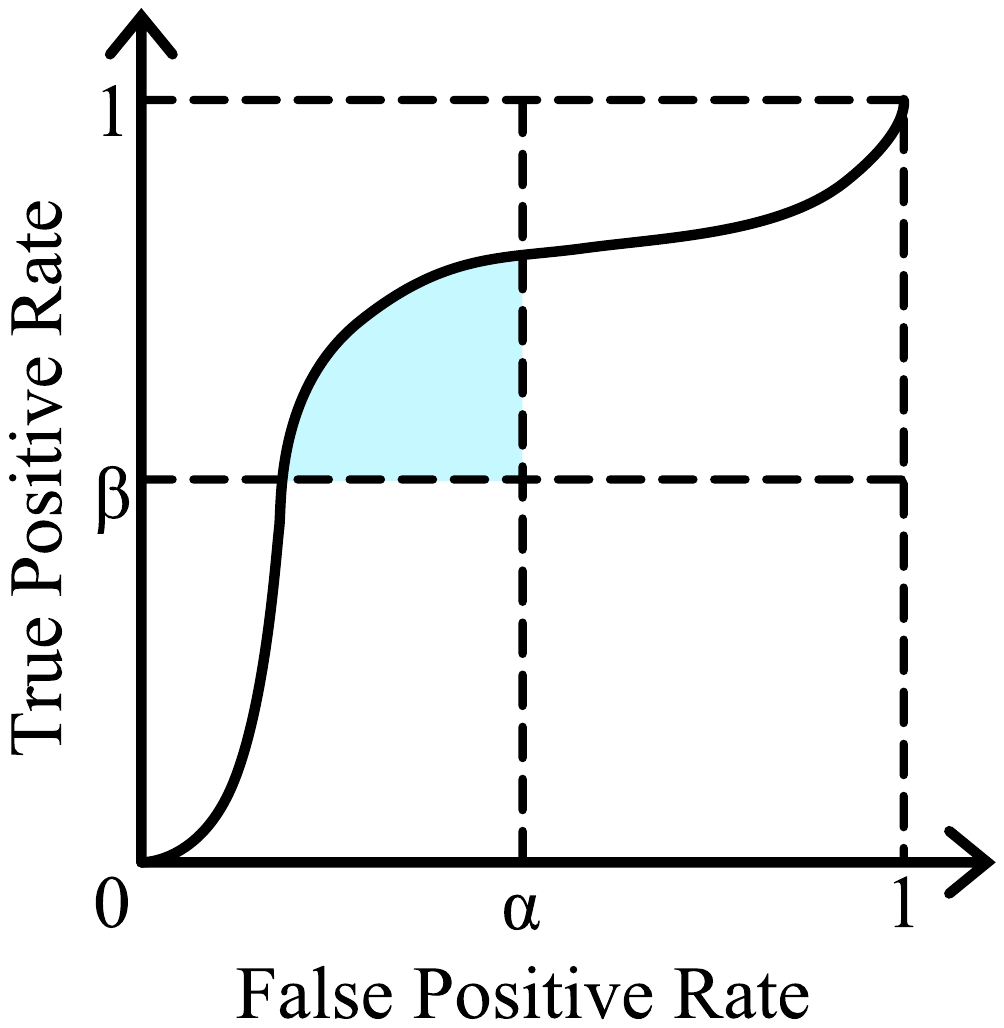}
	}
	\subfloat[rpAUC]{
		\label{fig:rpauc}
		\includegraphics[width=0.22\textwidth,trim=6.5cm 1.4cm 6.5cm 1cm]{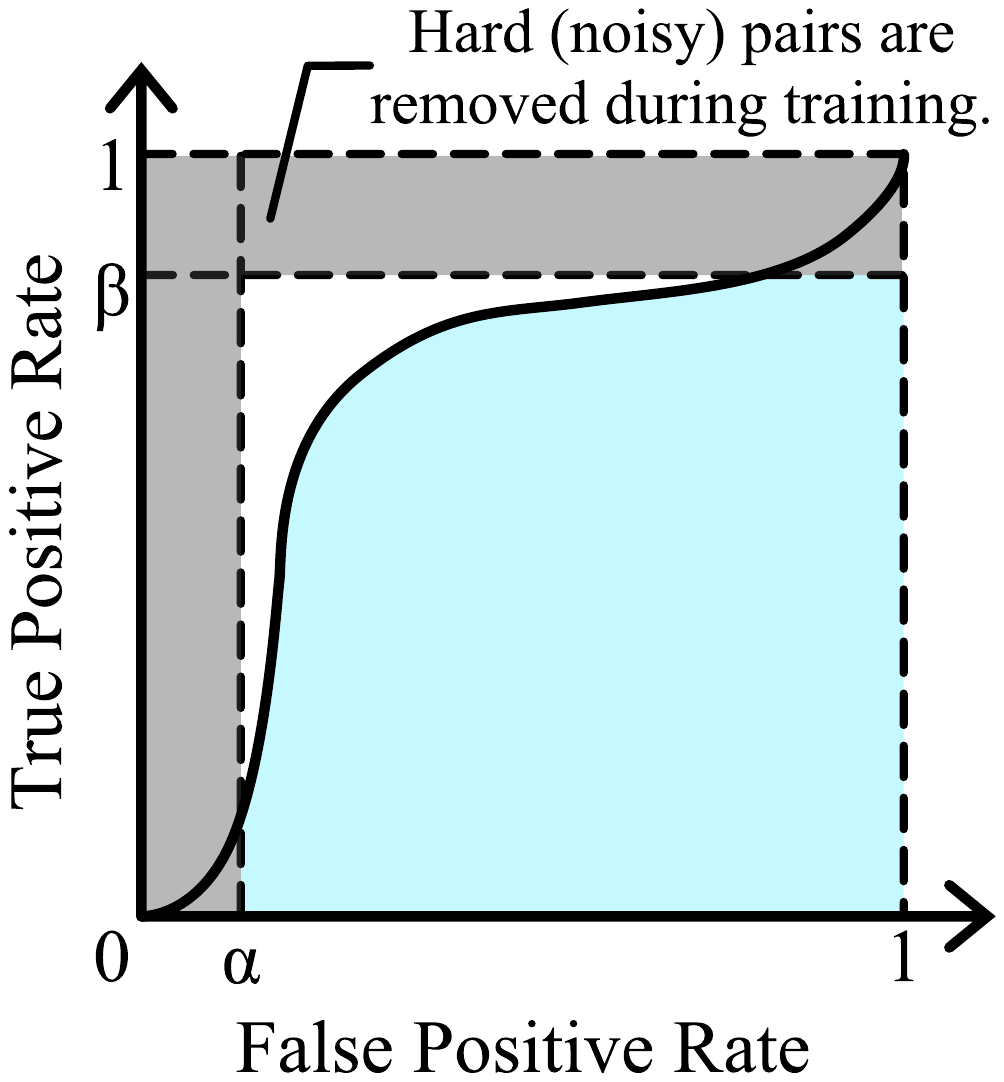}
	}
	\caption{Illustrations of ROC curves (black curves), AUC and Partial AUC variations (blue shades under the ROC curves). OPAUC and TPAUC are alternative performance measures specific to task needs. The proposed rpAUC aims for robust maximization of full AUC under weak supervision.}
	\label{fig:pauc-demo}
\end{figure}

\section{Related Work}\label{sec:related_work}
Weakly supervised learning (WSL) is a significant subfield of machine learning that tackles learning tasks with imperfect supervision. Unlike learning from fully and accurately labeled data, weak supervision can be incomplete, where only a subset of data is given with labels; inexact, where the data is given with only coarse-grind labels; or inaccurate, where the labels are not always true. For instance, semi-supervised learning~\cite{zhubook} and positive-unlabeled learning~\cite{pulsurvey} are common scenarios of learning with incomplete supervision. Multi-instance learning~\cite{milsurvey} and partial label learning~\cite{pll} belong to the inexact supervision category, whereas learning from inaccurate supervision encompasses label noise learning~\cite{Han2021nllsurvey} and crowdsourcing.

Among the topics mentioned, our work specifically relates to certain branches, i.e., unlabeled-unlabeled learning~\cite{uulearning} and noisy label learning (NLL)~\cite{Natarajan2013nll}. 
The proposed unified view shares a similar problem setup of unlabeled-unlabeled learning, but with the different objective of maximizing AUC{}. 
The property of noisy label AUC optimization problem have been preliminarily studied in literature~\cite{Menon2015,Ghosh2017}, which is related partially to the theoretical analysis of our framework.
The proposed robust method for weakly supervised AUC optimization through rpAUC maximization is closely related to the robust training approaches in noisy label learning.
We demonstrate that it has a similar mechanism to techniques such as the small loss trick~\cite{Liu2016,Jiang2018mentornet,Ren2018} and dynamic thresholding~\cite{Yang2021dythre}. 
These techniques remove or lower the weight of instances inducing high risk, and belong to sample selection in the context of noisy label learning~\cite{Song2022nllsurvey}. Other NLL approaches include robust surrogate losses~\cite{Zhang2018gce,Wang2019sce}, robust model architectures~\cite{Xiao_2015_CVPR,goldberger2017adaptlayer,han2018masking}, regularization methods~\cite{mixup,Tanno_2019_CVPR,lukasik2020labelsmooth}, etc. It is noteworthy that a key information, namely noise transition matrix, is usually required by the NLL approaches for mitigate the label noise problem. Since such information is usually unknown in practice, its estimation becomes a essential step for noisy label learning~\cite{R1,R2,R4}. In contrast, the proposed WSAUC does not require such knowledge for noisy label AUC optimization to achieve the consistency with the true risk.

Additionally, semi-supervised learning (SSL) is another major field of WSL, which aims to learn from both labeled and unlabeled data. Traditional SSL approaches include generative models~\cite{semi1994,nigam2000text,vae}, low-density separation-based models~\cite{tsvm,Chapelle2006,Li2010}, graph-based models~\cite{Blum2001,lpa}, and disagreement-based models~\cite{cotraining,tritraining}. With the rise of deep neural networks in recent years, SSL approaches based on consistency regularization~\cite{ladder,mixup,Miyato2019}, entropy minimization~\cite{ent_min}, graph neural networks~\cite{gnn,gcn,ggnn}, and others have demonstrated strong capabilities. 

Positive-unlabeled learning (PUL) is designed to learn from only positive and unlabeled data. Researchers have proposed various methods to address this problem, such as selecting reliable negative samples~\cite{Liu2002pul,Li2003pul}, reweighting the unlabeled samples~\cite{Lee2003pul,Ward2008,Elkan2008}, and more. Recently, approaches based on the unbiased risk estimators of PUL have proven to be effective~\cite{Plessis2014,Kiryo2017,Gong2021,Plessis2015}.

Multi-instance learning (MIL) tackles the challenge of learning from inexact supervision. In MIL, training instances are organized into bags, and labels are only given at the bag level. Instance-level approaches typically assume the bag label depends on the presence of any positive instance~\cite{Dietterich1997,mil,Carbonneau2016,Xiao2017}, while bag-level approaches assume the bag label is defined by multiple concepts or instances~\cite{Gaertner2002,Wei2017,Ilse2018,Wang2019}. In this work, we adopt the assumption of instance-level MIL{}.

Given the challenging nature of the task, only a limited number of studies have attempted to unify various weakly supervised learning scenarios. For instance, SAFEW~\cite{safew} focused on the safety aspect of WSL tasks, while LIISP~\cite{Zhangzy2022liisp} tackled the challenge of learning from incomplete and inaccurate supervision within a single framework. Meanwhile, CEGE~\cite{cege} proposed a general framework for WSL based on centroid estimation. DivideMix~\cite{R3} is a holistic model that address semi-supervised learning and noisy label learning simultaneously. However, these studies mainly concentrate on accuracy-oriented classification problems and cannot be readily applied to the AUC optimization problem. 

AUC is one of the most commonly used performance measure other than classification accuracy~\cite{auc}. It measures how well the model ranks the positive instances before the negative ones without considering the threshold, and therefore, it does not suffer from the inefficiency seen in accuracy when dealing with data imbalance. 
AUC optimization~\cite{cortes2003auc,aucsurvey} is the learning paradigm that aims to maximize the AUC performance of models. This approach makes the models more robust to data imbalance and leads to improved ranking performance.
Much effort has been devoted to AUC optimization problem from multiple aspects. In terms of theoretical properties, \citet{consistency} have studied the consistency of commonly used surrogate losses of AUC, while \citet{aucgeb} have studied the generalization bounds of AUC optimization models. In terms of online and stochastic optimization, \citet{opauc} have proposed an AUC optimization algorithm based on the covariance matrix, while \citet{solam} have converted the AUC optimization problem into a stochastic saddle point problem that can be optimized using stochastic gradient-based methods. \citet{Lei2021} have further proposed a stochastic proximal AUC maximization algorithm that achieves a linear convergence rate. Recently, AUC optimization with deep neural networks has gained increasing attention. \citet{Liu2020} have proposed deep AUC maximization based on pairwise AUC loss, while \citet{Yuan2022} have proposed an end-to-end AUC optimization method that solves a compositional objective. These research efforts have motivated many real-world applications of AUC optimization, including software build outcome prediction~\cite{ijcai18}, medical image classification~\cite{Azizi2021}, and abnormal behavior detection~\cite{Feizi2020}.

In contrast to the extensive research on accuracy-oriented weakly supervised learning, the study on the weakly supervised AUC optimization is rather few, and mostly restricted to specific situations of weakly supervised learning. \citet{puauc} study the semi-supervised AUC optimization based on positive-unlabeled learning. \citet{Xie2018} study unbiased risk estimator of the semi-supervised AUC optimization. \citet{Charoenphakdee2019} study the property of AUC optimization under label noise.
So far, no research has provided a unified approach for solving the AUC optimization problem across various weakly supervised learning scenarios.

\section{Preliminary}\label{sec:preliminary}

In supervised learning, we are provided with a dataset containing labeled data that has been sampled from a certain distribution, i.e.:
\begin{equation}
    \Xcall := {\{(\xv_i, y_i) \}}_{i=1}^{n}
	\simiid p(\xv, y)\,. 
\end{equation}
For convenience, we will refer to the subsets of positive and negative data in \(\Xcall\) as:
\begin{align*}
\Xcalp :=& {\{\xpi \}}_{i=1}^{n_P}
	\simiid
	p_P(\xv):=p(\xv\mid y=+1)\,, \text{\;and} \notag \\
\Xcaln :=& {\{\xnj \}}_{j=1}^{n_N}
	\simiid
	p_N(\xv):=p(\xv\mid y=-1)\,,
\end{align*}
so that \(\Xcall = \Xcalp \cup \Xcaln\).

Let \(f:\Xcal \rightarrow \Rcal\) be a model that is expected to output larger score for positive instances than negative ones. Given a classification threshold \(t\), we define the model's true positive rate \(\tpr_f(t)=\Pr(f(x)\ge t | y = 1)\), and the false positive rate \(\fpr_f(t)=\Pr(f(x)\ge t | y = 0)\). 
When \(t\) varies, the ROC curve (receiver operating characteristic curve) plots the TPR against FPR of the model \(f\), as illustrated in \cref{fig:auc}.
The AUC is defined as the area under the ROC curve:
\begin{equation}
\auc = \int_{-\infty}^{\infty} \tpr_f(t)\, \mathrm{d}\fpr_f(t).
\end{equation}

Since AUC is equivalent to the probability of a randomly drawn positive instance being ranked before a randomly drawn negative instance~\cite{auc,a1,a2}, the AUC of the model \(f\) can be formulated as:
\begin{equation}
\auc = 1-\EP[\EN[\elloi(f(\xp) - f(\xn))]]\,.\label{eq:auc}
\end{equation}
Here \(\elloi(z)=\Ibb[z<0]\).
Without introducing any ambiguity, we will use \(f(\xp, \xn)\) hereinafter to denote \(f(\xp) - f(\xn)\) for clarity.
Maximizing AUC is equivalent to minimizing the following AUC risk. To avoid confusion, we denote the true AUC risk as P-N AUC risk, since it measures the error rate of ranking positive instances before negative instances:
\begin{equation}
\Rpn(f) = \EP\left[\EN[\elloi(f(\xp, \xn))]\right]\,.\label{eq:rpn}
\end{equation}
In practice, we typically solve the following empirical risk minimization (ERM) problem with a finite sample:
\begin{equation}\label{eq:estimator:pn}
\min_f \quad \empRpn(f)=\frac{1}{|\Xcalp||\Xcaln|} \sum_{\xp \in \Xcalp}\sum_{\xn \in \Xcaln}\ell(f(\xp, \xn))\,.
\end{equation}

Recently, researchers have discovered that partial AUC (pAUC) is a useful metric for certain applications.
Partial AUC focuses on specific part of the area under the ROC curve.
There are two variations of pAUC{}: one-way pAUC (OPAUC~\cite{Narasimhan2013opauc}, c.f.\ \cref{fig:opauc}), which limits FPR in a range \([\alpha, \beta]\), and two-way pAUC (TPAUC~\cite{Yang2019tpauc}, c.f.\ \cref{fig:tpauc}), which limits \(\fpr < \alpha\) and \(\tpr > \beta\). In this paper, we focus on the two-way pAUC{}. 
It can be formulated as follows:
\begin{align}
  \mathrm{pAUC} &= 1- \Ebb_{\xp\sim p_P^-(\xv)}[\Ebb_{\xn\sim p_N^+(\xv)}[\elloi(f(\xp, \xn))]]\,,
\end{align}
where \(p_P^-(\xv) = p_P(\xv|f(\xv)\in(-\infty, \tpr^{-1}_f(\beta)])\) and \(p_N^+(\xv) = p_N(\xv|f(\xv)\in[\fpr^{-1}_f(\alpha), \infty))\) are difficult instances, i.e., positive instances with lower scores and negative instances with higher scores.
In \cref{sec:robust}, we develop a variation of TPAUC, namely reversed partial AUC (rpAUC), to achieve robust AUC optimization under weak supervision.

\section{The Unified View of WSAUC}\label{sec:unified}

\begin{table*}
  \caption{Summary of the connections between weakly supervised AUC optimization risks and supervised AUC risk \(\Rpn\).}
  \label{tab:summary}
  \small
  \centering
  \setlength{\tabcolsep}{0pt}
  \renewcommand{\arraystretch}{1.2}
  \begin{tabularx}{\textwidth}{      
    >{\hsize=.18\hsize}>{\hspace{2mm}\raggedright\arraybackslash}X
    >{\hsize=1.17\hsize}>{\centering\arraybackslash}X
    >{\hsize=.25\hsize}>{\centering\arraybackslash}X
    >{\hsize=.18\hsize}>{\centering\arraybackslash}X
    >{\hsize=.28\hsize}>{\centering\arraybackslash}X
    >{\hsize=.32\hsize}>{\centering\arraybackslash}X
    >{\hsize=.22\hsize}>{\centering\arraybackslash}X }
    \toprule
    Section    & Setting                                          & Data sets                   & \(R\)                         & \(a\)                      & \(b\)                                   & Risk def.              \\
    \midrule
    \S~\ref{sec:preliminary}& Supervised AUC Optimization                      & \(\Xcalp, \Xcaln\)                    & \(\Rpn\)                      & \(1\)                      & \(0\)                        & \cref{eq:rpn}          \\
    \S~\ref{subsec:nll}     & Noisy Label AUC Optimization                     & \(\Xcalpt, \Xcalnt\)                  & \(\risk{\tilde P \tilde N}\)  & \(1\!-\!\eta_P\!-\!\eta_N\)& \({(\eta_P\!+\!\eta_N)}/{2}\)& \cref{eq:risk:nll}     \\
    \S~\ref{subsec:pul}     & Positive-Unlabeled AUC Optimization              & \(\Xcalp, \Xcalu\)                    & \(\Rpu\)                      & \(\pi_N\)                  & \({\pi_P}/{2}\)              & \cref{eq:pu_risk}      \\
    \S~\ref{subsec:mil}     & Multi-Instance AUC Optimization                  & \(\Xcalpt, \Xcaln\)                   & \(\risk{\tilde P N}\)         & \(1\!-\!\eta_P\)           & \({\eta_P}/{2}\)             & \cref{eq:mil_risk}     \\
    \S~\ref{subsec:ssl}     & Semi-Supervised AUC Optimization                 & \(\Xcalp, \Xcalu, \Xcaln\)            & \(\risk{PNU}\)                & \(1\)                      & \(0\) (corrected)            & \cref{eq:risk:ssl}     \\
    \S~\ref{subsec:sslnl}   & Semi-Supervised AUC Optimization with Label Noise& \(\Xcalpt, \Xcalu, \Xcalnt\)          & \(\risk{\tilde P \tilde N U}\)& \(1\!-\!\eta_P\!-\!\eta_N\)& \({(\eta_P\!+\!\eta_N)}/{2}\)& \cref{eq:risk:sslnl}   \\
    \midrule
    \S~\ref{subsec:unified} & Unified Risk Formulation                         & \multicolumn{5}{c}{\(R\hspace{1.48em}=\hspace{1.48em} a \hspace{.9em} \Rpn \hspace{.9em}+\hspace{.9em} b\hspace{1.04em}\)}\\
    \bottomrule
  \end{tabularx}
\end{table*}

\subsection{The Unified Formulation}\label{subsec:unified}

The unified view of weakly supervised AUC optimization can be formulated based on the problem of minimizing AUC risk with two contaminated instance sets that have different class proportions. This learning problem can also be referred to as unlabeled-unlabeled learning~\cite{uulearning}, which focus on building classification models with minimal supervision, i.e., two unlabeled instances set with different class priors. Suppose we have two contaminated instance sets \(\Xcala\) and \(\Xcalb\), which can be regarded as samples from the mixture distributions of the positive and negative distributions with different proportions:
\begin{align*}
  \Xcala := {\{\xpi \}}_{i=1}^{n_A}
    \simiid
    &p_A(\xv):=\theta_A p_P(\xv) + (1-\theta_A)p_N(\xv)   \,; \\
  \Xcalb := {\{\xnj \}}_{j=1}^{n_B}
    \simiid
    &p_B(\xv):=\theta_B p_P(\xv) + (1-\theta_B)p_N(\xv)   \,,
\end{align*}
Without loss of generality, here we assume \(\theta_A > \theta_B\).
The AUC risk defined on the distribution \(p_A\) and \(p_B\) is:
\begin{equation}
  \Rab(f) := \Ebb_{\xp\sim p_A(\xv)}\left[\Ebb_{\xn\sim p_B(\xv)}[\elloi(f(\xp, \xn))]\right]\,, \label{eq:uni_emp}
\end{equation}
where \(f(\xp, \xn)\) is the shorthand for \(f(\xp) - f(\xn)\).
Then, the model can be obtained by solving the empirical risk minimization (ERM) problem:
\begin{equation}\label{eq:estimator:cf}
\min_f \quad \empRab(f)=\frac{1}{|\Xcala||\Xcalb|} \sum_{\xp \in \Xcala}\sum_{\xn \in \Xcalb}\ell(f(\xp, \xn))\,.
\end{equation}

The optimization problem is akin to treating both sets as positive and negative instance sets without accounting for their impurity. 
However, it can be shown that the impure AUC risk can be rewritten in a unified formulation which the true risk can be obtained with a linear transformation.

\begin{theorem}[unified formulation]\label{inaccurate}
  The impure AUC risk \(\Rab\) over two contaminated distributions \(p_A\) and \(p_B\)
  can be rewrote in the unified formulation as:
  \begin{equation}
    \Rab = a\Rpn + b \,,\label{eq:linear}\\
  \end{equation}
  where the bias term \(b=\sfrac{(1\!-\!a)}{2}\).
  Based on this formulation, the true risk \(\Rpn \) can be obtained from \(\Rab \) with a linear transformation.
\end{theorem}
\begin{proof}
  For any $f$, we can rewrite the risk:
  {
  \allowdisplaybreaks
  \begin{align*}
       &\Rab(f)\\
      ={}&\underset{\xp \in \mathcal{X}_{A}}{\Ebb}\left[\underset{\xn \in \mathcal{X}_{B}}{\Ebb}\left[\ell\left(f(\xp, \xn)\right)\right]\right]\\
      ={}&\underset{\xp \in \mathcal{X}_{A}}{\Ebb}\left[(1\!-\!\theta_B)\!\underset{\xn \in \mathcal{X}_{N}}{\Ebb}\!\left[\ell\left(f(\xp, \xn)\right)\right]\!+\! \theta_B\!\underset{\xn \in \mathcal{X}_{P}}{\Ebb}\!\left[\ell\left(f(\xp, \xn)\right)\right]\right]\\
      ={}&\theta_A(1\!-\!\theta_B)
      \underset{\xp \in \mathcal{X}_{P}}{\Ebb}\left[\underset{\xn \in \mathcal{X}_{N}}{\Ebb}\left[\ell\left(f(\xp, \xn)\right)\right]\right]\\
      &\quad +\theta_A\theta_B
      \underset{\xp \in \mathcal{X}_{P}}{\Ebb}\left[\underset{\xn \in \mathcal{X}_{P}}{\Ebb}\left[\ell\left(f(\xp, \xn)\right)\right]\right]\\
      &\quad +(1\!-\!\theta_A)(1\!-\!\theta_B)
      \underset{\xp \in \mathcal{X}_{N}}{\Ebb}\left[\underset{\xn \in \mathcal{X}_{N}}{\Ebb}\left[\ell\left(f(\xp, \xn)\right)\right]\right]\\
      &\quad +(1\!-\!\theta_A)\theta_B
      \underset{\xp \in \mathcal{X}_{N}}{\Ebb}\left[\underset{\xn \in \mathcal{X}_{P}}{\Ebb}\left[\ell\left(f(\xp, \xn)\right)\right]\right]\\
      ={}&\theta_A(1-\theta_B)\Rpn (f)+\theta_A\theta_B \Rpp (f)\\
      &\quad +(1-\theta_A)(1-\theta_B)\Rnn (f)+(1-\theta_A)\theta_B \risk{NP}(f)\\
      ={}&\theta_A(1-\theta_B)\Rpn (f)+(1-\theta_A)\theta_B(1-\Rpn )\\
      &\quad +\frac{1}{2}(2\theta_A\theta_B+1-\theta_A-\theta_B)\\
      ={}&(\theta_A-\theta_B)\Rpn (f)+\frac{1-(\theta_A-\theta_B)}{2}\,.
  \end{align*}
}

  This completes the proof.
\end{proof}
\begin{corollary}[consistency of inaccurate case]
  In the AUC optimization problem described above, suppose \(f^*\) is a minimizer of the impure AUC risk \(\Rab\) over two contaminated distributions \(p_A\) and \(p_B\), i.e., \(f^* = \argmin \Rab\), then \(f^*\)  is also a minimizer of the true AUC risk \(\Rpn \), i.e., \(\Rab\) is consistent with \(\Rpn \).
\end{corollary}
\begin{proof}
  According to \cref{inaccurate}, for any $f$ we have
  $$
  \Rab(f) = a \Rpn (f)+\frac{1-a}{2}\,,
  $$
  where $a>0$.
  Thus, for any $f$,
  $$
  \Rpn (f^*)-\Rpn (f)=\frac{\Rab(f^*)-\Rab(f)}{a}\leq 0
  $$
  because $a>0$ and $f^*$ is a minimizer of $\Rab$. 
  This proves that $f^*$ is also a minimizer of $\Rpn $.
\end{proof}

In the rest of this section, we will demonstrate how different weakly supervised AUC optimization problems can be connected with this unified formulation.
This provides a simple way for handling AUC optimization problem in almost all common case of weakly supervised AUC optimization tasks by simply adopting any pairwise risk minimization algorithm.

It is noteworthy that when optimizing for accuracy, it is required to know the mixing proportions of the distributions (i.e., \(\theta_A\) and \(\theta_B\)) in order to correct estimation bias and achieve consistency. In AUC optimization, however, we can maintain statistical consistency even without knowledge of the mixing proportions.
Nevertheless, the presence of impurity can still have an impact on model learning with finite data, which is depicted in the coefficient \(a\) in \cref{eq:linear}. Directly solving the optimization problem above may be less robust when there is significant impurity in the instance sets. In \cref{sec:robust}, we address this issue by minimizing the empirical risk of a novel variety of partial AUC{}.

\subsection{Specialization to Common Cases}

\subsubsection{Noisy Label AUC Optimization}\label{subsec:nll}
We begin by discussing AUC optimization under label noise, which is the most common case of inaccurate supervision. In binary classification with label noise, an instance label may be flipped with a certain probability.  Consider the asymmetrical noise for two classes, a positive instance may be mislabeled with a probability of \(\eta_P\), and a negative instance may be mislabeled with a probability of \(\eta_N\). 
Such a noisy label AUC optimization problem can be easily converted to the problem defined in \cref{eq:uni_emp} by setting the noise ratio as the mixture proportion:
\begin{align*}
  \Xcal_{\tilde P} := {\{\xpi \}}_{i=1}^{n_P}
    \simiid
    &p_{\tilde P}(\xv):=(1-\eta_P) p_P(\xv) + \eta_P p_N(\xv)   \,;\\
  \Xcal_{\tilde N} := {\{\xnj \}}_{j=1}^{n_N}
    \simiid
    &p_{\tilde N}(\xv):=\eta_N p_P(\xv) + (1-\eta_N) p_N(\xv)   \,.
\end{align*}
The problem of learning with noisy labels can be seen as a variant of learning from two contaminated sets or unlabeled-unlabeled learning. The distinctions between these two problem formulations are discussed in \citet{uulearning}. Specifically, learning with noisy labels assumes a noisy rate of \(\eta_P+\eta_N < 0.5\), while the class prior remains the same regardless of the presence of noise. Without this assumption, the marginal distribution, 
\(p(\xv)\), may change, which would require addressing the problem under a covariate shift assumption.

By simply substituting the two noisy distributions for \(\Xcala\) and \(\Xcalb\) in \cref{eq:uni_emp}, we obtain the noisy AUC risk:
\begin{equation}\label{eq:risk:nll}
  \risk{\tilde{P}\tilde{N}}(f) := \Ebb_{\xp\sim p_{\tilde P}(\xv)}\left[\Ebb_{\xn\sim p_{\tilde N}(\xv)}[\elloi(f(\xp,\xn))]\right]\,. 
\end{equation}
And by letting \(\theta_A=1-\eta_P\) and \(\theta_B=\eta_N\), its easy to show the following corollary:
\begin{corollary}
  \(\risk{\tilde{P}\tilde{N}}\) is consistent with the AUC risk over the true distribution \(\Rpn \), and
  \begin{gather*}
    \risk{\tilde{P}\tilde{N}} = \tilde{a} \Rpn +  \frac{1-\tilde{a}}{2}\,,
    \\\tilde{a} = 1-\eta_P -\eta_N\,.
  \end{gather*}
\end{corollary}
This indicates that the noisy AUC risk is still consistent with the clean AUC risk. Empirically, we can solve the following ERM problem:
\begin{equation}\label{eq:estimator:nll}
  \min_f \quad \emprisk{\tilde{P}\tilde{N}}(f) = \frac{1}{|\Xcalpt||\Xcalnt|}\sum_{\xp\in \Xcalpt}\sum_{\xn\in \Xcalnt}\ell(f(\xp,\xn))\,. 
\end{equation}

\subsubsection{Positive-Unlabeled AUC Optimization}\label{subsec:pul}
We next discuss positive-unlabeled AUC optimization, in which only supervision of one class is available.
Suppose the underlying class prior probabilities of positive and negative class are \(\pi_P\) and \(\pi_N\), in this case, the instances with positive labels can be regarded as a pure positive set \(\Xcalp\) with \(\theta_A=1\), and the unlabeled data consists an contaminated set \(\Xcalu\) with \(\theta_B=\pi_P\). Then the two sets can be formulated as:
\begin{align*}
  \Xcalp := {\{\xpi \}}_{i=1}^{n_P}
    \simiid
    &p_P(\xv)   \,, \text{\;and} \\
  \Xcalu := {\{\xnj \}}_{j=1}^{n_U}
    \simiid
    &p_U(\xv):=\pi_P p_P(\xv) + \pi_N p_N(\xv)   \,.
\end{align*}
Then we define P-U AUC risk on the positive and unlabeled set:
\begin{equation}
  \Rpu (f) := \Ebb_{\xp\sim p_P(\xv)}\left[\Ebb_{\xn\sim p_U(\xv)}[\elloi(f(\xp,\xn))]\right]\,. \label{eq:pu_risk}
\end{equation}
And it is easy to show that:
\begin{corollary}
  The P-U AUC risk \(\Rpu \) is consistent with the true AUC risk \(\Rpn \), and
\begin{align*}
  \Rpu  &= \pi_N \Rpn +  \frac{\pi_P}{2}\,.
\end{align*}
\end{corollary}
Such a formulation connects the positive-unlabeled AUC optimization problem with the noisy label case by treating one set as having the smallest noise rate and the other set as having the largest noise ratio (as noisy as the marginal distribution \(p(\xv)\)). It shows that the AUC optimization with one-sided supervision can be solved by treating the unlabeled data as negative data.
Practically, we solve the positive-unlabeled AUC optimization problem by minimizing the following empirical risk:
\begin{equation}\label{eq:estimator:pu}
  \min_f \quad \empRpu(f)=\frac{1}{|\Xcalp||\Xcalu|} \sum_{\xp \in \Xcalp}\sum_{\xn \in \Xcalu}\ell(f(\xp, \xn))\,.
\end{equation}

\subsubsection{Multi-Instance AUC Optimization}\label{subsec:mil}
We demonstrate how our unified view can be applied to multi-instance AUC optimization, a learning problem with inexact, bag-level supervision. 
In this case, we have a set of positive bags \(S_P = \{B^+_i\}_{i=1}^{N_P}\) and a set of negative bags \(S_N = \{B^-_i\}_{j=1}^{N_N}\). Each positive bag \(B^+_i\) contains at least one positive instance, while the negative bag \(B^-_j\) contains none of positive instances. Following the formulation of the prior research~\cite{cege}, we regard the instance occurred in the negative bags as they are sampled from the pure negative distribution \(p_N\), while the instances occurred in the positive bags as from some mixture distribution \(p_{\tilde P}\) of the positive and negative distributions with some unknown proportion \(1-\eta_P\) and \(\eta_P\). Then, the AUC risk over the positive bag instances and negative bag instances can be defined as:
\begin{equation}
  \risk{\tilde{P}N}(f) := \Ebb_{\xp\sim p_{\tilde P}(\xv)}\left[\Ebb_{\xn\sim p_N(\xv)}[\elloi(f(\xp,\xn))]\right]\,. \label{eq:mil_risk}
\end{equation}
\begin{corollary}
  The multi-instance AUC risk \(\risk{\tilde{P}N}\) is consistent with the true AUC risk \(\Rpn \), and
\begin{align*}
  \risk{\tilde{P}N} &= (1-\eta_P) \Rpn +  \frac{\eta_P}{2}\,.
\end{align*}
\end{corollary}
The above corollary provides us a way to handle the AUC optimization in multi-instance learning. 
Practically, we need to first construct the instance sets \(\Xcalpt\) and \(\Xcaln\) by taking the union of the instance bags:
\begin{equation}
  \Xcalpt = \bigcup_{i=1}^{N_P} B^+_i \,, \quad \Xcaln = \bigcup_{j=1}^{N_N} B^-_j\,,
\end{equation}
and solve the following ERM problem:
\begin{equation}\label{eq:estimator:mil}
  \min_f \quad \emprisk{\tilde P N}(f)=\frac{1}{|\Xcalpt||\Xcaln|} \sum_{\xp \in \Xcalpt}\sum_{\xn \in \Xcaln}\ell(f(\xp, \xn))\,.
\end{equation}
This approach produces scores on instance level. To obtain bag-level score, one can simply calculate the maximum of the instance scores in the bags.

\subsubsection{Semi-Supervised AUC Optimization}\label{subsec:ssl}
Building a model with limited labeled data and a relatively large amount of unlabeled data is a common scenario in semi-supervised learning. In contrast to the previous cases, the data can be divided into three sets based on their labels:
\begin{alignat*}{4}
  &\Xcalp := &&{\{\xpi \}}_{i=1}^{n_P}&&\simiid &&p_P(\xv)   \,,  \\
  &\Xcaln := &&{\{\xnj \}}_{j=1}^{n_N}&&\simiid &&p_N(\xv)   \,, \text{\;and} \\
  &\Xcalu := &&{\{\xuk \}}_{k=1}^{n_U}&&\simiid &&p_U(\xv):=\pi_P p_P(\xv) + \pi_N p_N(\xv)   \,.
\end{alignat*}
Symmetry to the P-U AUC risk (\cref{eq:pu_risk}), we can define U-N AUC risk over unlabeled and negative data:
\begin{equation}
  \Run (f) := \Ebb_{\xp\sim p_U(\xv)}\left[\Ebb_{\xn\sim p_N(\xv)}[\elloi(f(\xp,\xn))]\right]\,. \label{eq:un_risk}
\end{equation}

By combining each two the three data sets, we have three risk terms: \(\Rpn\), \(\Rpu\), and \(\Run\). Minimizing each risk term corresponds to a subproblem in the form of \cref{eq:estimator:cf}. However, the risk terms calculated with unlabeled data, i.e., \(\Rpu\) and \(\Run\), are biased.
Next, we show that, with a proper combination of all the three terms, the unbiasedness can be maintained while the unlabeled data is fully utilized in the learning process.

Firstly, it can be shown that that by combining the P-U AUC risk and U-N AUC risk, the summation risk yields an unbiased risk estimator even if the class prior probabilities are unknown~\cite{Xie2018}. 
\begin{theorem}
  The sum of \(\Rpu \) and \(\Run \) is consistent with the true AUC risk \(\Rpn \), and the bias is always \(\frac{1}{2}\).
\begin{align}
  \Rpu  + \Run  - \frac{1}{2} &= \Rpn \,. \label{eq:pnnu}
\end{align}
\end{theorem}
\begin{proof}
  For any $f$, we have
      $$
      \begin{aligned}
          &\Rpu (f) + \Run (f) \\
          =& \underset{\xp \in \mathcal{X}_{P}}{\Ebb}\left[\underset{\xn \in \mathcal{X}_{U}}{\Ebb}\left[\ell\left(f(\xp, \xn)\right)\right]\right]+ \underset{\xp \in \mathcal{X}_{U}}{\Ebb}\left[\underset{\xn \in \mathcal{X}_{N}}{\Ebb}\left[\ell\left(f(\xp, \xn)\right)\right]\right]\\
          =&  \underset{\xp \in \mathcal{X}_{P}}{\Ebb}\left[\pi_N \underset{\xn \in \mathcal{X}_{N}}{\Ebb}\left[\ell\left(f(\xp, \xn)\right)\right] + \pi_P\underset{\xn \in \mathcal{X}_{P}}{\Ebb}\left[\ell\left(f(\xp, \xn)\right)\right]\right]\\
           &+  \underset{\xn \in \mathcal{X}_{N}}{\Ebb}\left[\pi_P \underset{\xp \in \mathcal{X}_{P}}{\Ebb}\left[\ell\left(f(\xp, \xn)\right)\right] + \pi_N\underset{\xp \in \mathcal{X}_{N}}{\Ebb}\left[\ell\left(f(\xp, \xn)\right)\right]\right]\\
          =& (\pi_P + \pi_N) \Rpn (f) + \pi_P \Rpp (f) + \pi_N \Rnn (f)\\
          =& \Rpn (f) + \frac{1}{2}\,.
      \end{aligned}
      $$
      Thus we obtain the theorem.
\end{proof}
This also indicates that in the semi-supervised scenario, we can achieve unbiased AUC risk estimation by subtracting a \(\frac{1}{2}\) without knowing the class prior.

Secondly, to fully exploit the data for reducing the estimation variance, instead of directly minimizing \cref{eq:pnnu} through ERM, we define the following risk:
\begin{equation} \label{eq:risk:ssl}
  \Rpnu = \gamma \Rpn + (1-\gamma)(\Rpu + \Run - \frac{1}{2})\,,
\end{equation} 
where \(\gamma\) is the weighting coefficient. 

To calculate the empirical risk \(\empRpnu\), we need to sum up the pairwise losses over three data set pairs: \(\Xcalp \times \Xcaln\), \(\Xcalp \times \Xcalu\), and \(\Xcalu \times \Xcaln\). The bias induced by \((\Rpu + \Run)\) is constantly \(\sfrac{1}{2}\), which can be compensated by subtracting it from the empirical risk.
By doing so, the empirical risk \(\empRpnu\) becomes an unbiased risk estimator of the true AUC risk. Practically, whether to compensate the bias does not affect the training of the model, so we can simply ignore it.

\subsubsection{Semi-Supervised AUC Optimization with Label Noise}\label{subsec:sslnl}
We further consider the scenario where the available supervision is inaccurate and incomplete. Such a situation typically involves learning from a relatively small number of inaccurately labeled instances and a set of unlabeled instances. Let \(\eta_P\) and \(\eta_N\) denote the probabilities of a positive and a negative label being incorrect, respectively, and let \(\pi_P\) and \(\pi_N\) denote the class prior probabilities of the true distribution, with \(\pi_P + \pi_N = 1\).
It is noteworthy that we do not require knowledge of these values. The instances can then be categorized into three sets: a noisy positive set \(\Xcalpt \sim p_{\tilde P}\), a noisy negative set \(\Xcalnt \sim p_{\tilde N}\), and an unlabeled set \(\Xcalu \sim p_U\). 

Similar to the former case, by combining the noisy version of the P-U AUC risk and U-N AUC risk, the biased risk yields the same attenuation coefficient as in the noisy case.
\begin{corollary}\label{II}
The sum of \(\risk{\tilde{P}U}\) and \(\risk{U\tilde{N}}\) is consistent with the true AUC risk \(\Rpn \).
\begin{gather*}
  {R}_{\tilde{P}U} + {R}_{U\tilde{N}} - \frac{1}{2} = \tilde{a}  \Rpn  + \frac{1-\tilde{a} }{2}\,,
  \\\tilde{a}  = 1-\eta_P -\eta_N \,.
\end{gather*}
\end{corollary}

To reduce the variance of the risk estimator, we define the following risk:
\begin{equation}\label{eq:risk:sslnl}
  \risk{\tilde{P}\tilde{N}U} = \gamma \risk{\tilde{P}\tilde{N}} + (1-\gamma)(\risk{\tilde{P}U} + \risk{U\tilde{N}} - \frac{1}{2})\,.
\end{equation}
And the ERM problem can be solved like the previous semi-supervised case.

\subsubsection{Summary}
In this section, we show that various scenarios of weakly supervised AUC optimization problems can be addressed by minimizing the sum or weighted average of one or more AUC risk terms over instance set pairs. This type of optimization problem can be easily solved by incorporating additional instance pairs into a standard pairwise AUC optimization algorithm. The settings discussed are summarized in \cref{tab:summary}. In \cref{sec:theory} we introduce the theoretical analysis of the framework. In \cref{sec:robust}, we further propose a robust learning solution for the problems.

\section{Theoretical Analyses}\label{sec:theory}
In this section, we theoretically analyze the proposed risk functions, 
which consistent with the AUC risk over the true distribution.
Briefly, we (1) prove the excess risk bounds for the general case and inaccurate and incomplete case, which can be easily applied to all of the WSL scenarios discussed above, and
(2) discuss the variance reduction of the incomplete supervised learning scenario, showing that by introducing the unlabeled data, we can achieve better risk estimation with lower variance.

Here, we consider a kernel $K$ over $\mathcal{X}^2$, a strictly positive real number $C_w$. 
Let $ \mathcal{F}_{K} $ be a class of functions:
$$
\mathcal{F}_{K} = \{f_w:\mathcal{X}\to R,f_w(x)=K(w,x)|\|w\|_K\leq  C_w \}\,,
$$
where $\|x\|_K=\sqrt{K(x,x)}$.
We also assume that the surrogate loss $\ell$ is $L$-Lipschitz continuous, bounded by a strictly positive real number $C_{\ell}$,
and satisfies inequality $\ell \geq \ell_{01}$. For example, the squared loss and exponential loss satisfy these condition.

\subsection{Excess Risk}
In this part, we prove the excess risk bounds when we minimize the proposed risk functions.

Denote by $\hat f^*_{AB}$ the minimizer of empirical risk $\empRab(f)$, 
we introduce the following excess risk bound, showing that the risk of $\hat f^*_{AB}$
converges to risk of the optimal decision function in the function family $\mathcal{F}_{K}$.

\begin{theorem}[Excess Risk of General Case]\label{er1}
    Assume that $\hat f^*_{AB} \in \mathcal{F}_{K}$ is the minimizer of empirical risk $\empRab(f)$,
    $f^*_{PN} \in \mathcal{F}_{K}$ is the minimizer of true risk $\Rpn (f)$.
    For any $\delta>0$, with the probability at least $1-\delta$, we have
    $$
    \Rpn (\hat  f^*_{AB})-\Rpn (f^*_{PN}) \leq \frac{h(\delta)}{a}\sqrt{\frac{n_A+n_B}{n_A n_B}}\,,
    $$
    where $h(\delta)=8\sqrt{2}C_\ell C_w C_x+5\sqrt{2\ln{(2/\delta)}}$, and $n_A, n_B$ is the size of sampled contaminated instance set.
\end{theorem}

\begin{proof}
    Let $R^\prime_{AB}(f)=\frac{\Rab-\frac{1-a}{2}}{a}$ denote the linear transformation of $\Rab$ to estimate $\Rpn $,
    and $\hat R^\prime_{AB}(f)$ denote its empirical estimation.
    The excess risk of optimizing $\empRab(f)$ can be write as
    \begin{equation}\label[ineq]{eq1}
    \begin{aligned}
        &\Rpn (\hat  f^*_{AB})-\Rpn (f^*_{PN}) \\
        &\quad  = \Rpn (\hat  f^*_{AB})-\hat R^\prime_{AB}(\hat  f^*_{AB})+\hat  R^\prime_{AB}(\hat  f^*_{AB})\\
        &\quad \qquad \quad -\hat  R^\prime_{AB}(f^*_{PN})+\hat  R^\prime_{AB}(f^*_{PN})-\Rpn (f^*_{PN})\\
        &\quad \leq 2\max_{f \in \mathcal F}|\hat R^\prime_{AB}(f)-\Rpn (f)|\,.
    \end{aligned}
    \end{equation}
    Due to \cref{inaccurate}, the right term can be write as
    \begin{equation}\label{eq2}
        \max_{f \in \mathcal F}|\hat R^\prime_{AB}(f)-\Rpn (f)|=\max_{f \in \mathcal F}|\hat R^\prime_{AB}(f)-R^\prime_{AB}(f)|\,.
    \end{equation}
    According to Theorem 6 in \citet{usunier2005data}, for any $\delta>0$, 
    with probability at least $1-\delta$ for any $f \in \mathcal{F}_{K}$:
    \begin{equation}\label[ineq]{eq3}
    \begin{aligned}
        &\max_{f \in \mathcal F}|\hat R_{AB}(f)-R_{AB}(f)|\\
        & \leq 4 \sqrt{2} C_\ell C_w C_x \sqrt{\frac{n_A\!+\!n_B}{n_A n_B}}+5 \sqrt{\frac{n_A\!+\!n_B}{2 n_A n_B} \ln (2 / \delta)}\,,
    \end{aligned}
    \end{equation}
    where $C_x=\max{(\max_i{\|x_i\|},\max_j{\|x_j^\prime\|})}$.
    For convenience, we define
    $$
    h(\delta)=8\sqrt{2}C_\ell C_w C_x+5\sqrt{2\ln{(2/\delta)}}\,.
    $$
    So we have
    \begin{equation}\label[ineq]{gen_bound}
        \max_{f \in \mathcal F}|\hat R^\prime_{AB}(f)-R^\prime_{AB}(f)|\leq \frac{h(\delta)}{2a}\sqrt{\frac{n_A+n_B}{n_A n_B}}\,.
    \end{equation}
    Applying \cref{eq2} and \cref{gen_bound} to the right term in \cref{eq1}, we obtain the theorem.
\end{proof}

\cref{er1} guarantees that the excess risk of general case can be bounded plus the confidence term of order
$$
\mathcal{O}\left(\frac{1}{a \sqrt {n_{A}}}+\frac{1}{a \sqrt {n_{B}}}\right)\,.
$$

Denote by $\hat f^*_{\tilde P \tilde N U}$ the minimizer of the empirical risk $\emprisk{\tilde P \tilde N U}(f)$, 
similarly, we have a theorem to show that the risk of $\hat f^*_{\tilde P \tilde N U}$
converges to risk of the optimal decision function in the function family $\mathcal{F}_{K}$.

\begin{theorem}[Excess Risk of Inaccurate and Incomplete Case]\label{er2}
    Assume that $\hat f^*_{\tilde P \tilde N U} \in \mathcal{F}_{K}$ is the minimizer of empirical risk $\emprisk{\tilde P \tilde N U}(f)$,
    $f^*_{PN} \in \mathcal{F}_{K}$ is the minimizer of true risk $\Rpn (f)$.
    For any $\delta>0$, with the probability at least $1-\delta$, we have
    $$
    \begin{aligned}
        &\Rpn (\hat  f^*_{\tilde P \tilde N U})-\Rpn (f^*_{PN})\\
        & \leq \frac{h(\frac{\delta}{3})}{\tilde{a}}\!\left(\!{\gamma}\sqrt{\frac{n_{\tilde P}\!+\!n_{\tilde N}}{n_{\tilde P} n_{\tilde N}}}\!+\!{(1\!-\!\gamma)}\!\Big(\sqrt{\frac{n_{\tilde P}\!+\!n_{U}}{n_{\tilde P} n_{U}}} \!+\!\sqrt{\frac{n_{U}\!+\!n_{\tilde N}}{n_{U} n_{\tilde N}}}\Big)\!\right) ,
    \end{aligned}
    $$
    where $h(\delta)=8\sqrt{2}C_\ell C_wC_x+5\sqrt{2\ln{(2/\delta)}}$, and $n_{\tilde P}, n_{\tilde N}$ is the size of sampled contaminated instance set.
\end{theorem}

\begin{proof}
    Let $R^\prime_{\tilde P \tilde N U}(f)=\gamma\frac{{R}_{\tilde{P}\tilde{N}}-\frac{1-\tilde a}{2}}{\tilde a}+(1\!-\!\gamma)\frac{{R}_{\tilde{P}U} + {R}_{U\tilde{N}} - \frac{1}{2}-\frac{1-\tilde a}{2}}{\tilde a}$ denote the linear transformation of $\risk{\tilde P \tilde N U}$ to estimate $\Rpn $,
    and $\hat R^\prime_{\tilde P \tilde N U}(f)$ denote its empirical estimation.
    Similarly to \cref{eq1}, the excess risk of optimizing $\emprisk{\tilde P \tilde N U}(f)$ can be write as
    \begin{equation}\label[ineq]{eq1'}
    \begin{aligned}
        &\Rpn (\hat  f^*_{\tilde P \tilde N U})-\Rpn (f^*_{PN})\\
        &\quad = \Rpn (\hat  f^*_{\tilde P \tilde N U})-\hat R^\prime_{\tilde P \tilde N U}(\hat  f^*_{\tilde P \tilde N U})+\hat  R^\prime_{\tilde P \tilde N U}(\hat  f^*_{\tilde P \tilde N U})\\
        &\quad \qquad-\hat  R^\prime_{\tilde P \tilde N U}(f^*_{PN})+\hat  R^\prime_{\tilde P \tilde N U}(f^*_{PN})-\Rpn (f^*_{PN})\\
        &\quad \leq 2\max_{f \in \mathcal F}|\hat R^\prime_{\tilde P \tilde N U}(f)-\Rpn (f)| \,.
    \end{aligned}
    \end{equation}

    Due to \cref{II}, the right term can be write as
    \begin{equation}\label{eq2'}
        \max_{f \in \mathcal F}|\hat R^\prime_{\tilde P \tilde N U}(f)\!-\!\Rpn (f)|=\max_{f \in \mathcal F}|\hat R^\prime_{\tilde P \tilde N U}(f)\!-\!R^\prime_{\tilde P \tilde N U}(f)|\,.
    \end{equation}
    Respectively, replace $x\in \mathcal{X}_A, x^\prime \in \mathcal{X}_B$  by 
    $x\in \mathcal{X}_{\tilde P}, x^\prime \in \mathcal{X}_{\tilde N}$,
    $x\in \mathcal{X}_{\tilde P}, x^\prime \in \mathcal{X}_U$, and
    $x\in \mathcal{X}_U, x^\prime \in \mathcal{X}_{\tilde N}$ in \cref{gen_bound},
    we have that for any $\delta>0$, 
    with probability at least $1-\delta$ for any $f \in \mathcal{F}_{K}$:
    $$
    \max_{f \in \mathcal F}|\hat R^\prime_{\tilde P\tilde N}(f)-R^\prime_{\tilde P\tilde N}(f)|\leq \frac{h(\delta)}{2  \tilde a}\sqrt{\frac{n_{\tilde P}+n_{\tilde N}}{n_{\tilde P} n_{\tilde N}}}\,,
    $$
    $$
    \max_{f \in \mathcal F}|\hat R^\prime_{\tilde PU}(f)-R^\prime_{\tilde PU}(f)|\leq \frac{h(\delta)}{2 \tilde a}\sqrt{\frac{n_{\tilde P}+n_{U}}{n_{\tilde P} n_{U}}}\,,
    $$
    $$
    \max_{f \in \mathcal F}|\hat R^\prime_{U\tilde N}(f)-R^\prime_{U\tilde N}(f)|\leq \frac{h(\delta)}{2 \tilde a}\sqrt{\frac{n_{U}+n_{\tilde N}}{n_{U} n_{\tilde N}}}\,.
    $$
    Simple calculation showed that for any $\delta^\prime >0$, with probability at least $1-\delta^\prime$, we have
    \begin{equation}\label[ineq]{gen_bound'}
        \begin{aligned}
            &\max_{f \in \mathcal F}|\hat R^\prime_{\tilde P \tilde N U}(f)-R^\prime_{\tilde P \tilde N U}(f)|\\
            {}\leq{} & \gamma\left(\max_{f \in \mathcal F}|\hat R^\prime_{\tilde P \tilde N}(f)-R^\prime_{\tilde P \tilde N}(f)| \right)\\
            & \quad + (1-\gamma)\left(\max_{f \in \mathcal F}|\hat R^\prime_{\tilde P U}(f)-R^\prime_{\tilde P U}(f)| \right) \\
            & \quad + (1-\gamma)\left(\max_{f \in \mathcal F}|\hat R^\prime_{U \tilde N}(f)-R^\prime_{U \tilde N}(f)| \right)\\
            {}\leq{} & \frac{h(\frac{\delta^\prime}{3})}{2\tilde a}\!\!\left(\!{\gamma}\!\sqrt{\frac{n_{\tilde P}\!+\!n_{\tilde N}}{n_{\tilde P} n_{\tilde N}}}\!+\!{(1\!-\!\gamma)}\!\Big(\sqrt{\frac{n_{\tilde P}\!+\!n_{U}}{n_{\tilde P} n_{U}}} \!\!+\!\!\sqrt{\frac{n_{U}\!+\!n_{\tilde N}}{n_{U} n_{\tilde N}}}\Big)\!\right).
        \end{aligned}
    \end{equation}
    
    Applying \cref{eq2'} and \cref{gen_bound'} to the right term in \cref{eq1'}, we obtain the theorem.
\end{proof}

\cref{er2} guarantees that the excess risk of inaccurate and incomplete case can be bounded plus the confidence term of order
$$
\mathcal{O}\left(\frac{1}{\tilde a \sqrt {n_{\tilde P}}}+\frac{1}{\tilde a \sqrt {n_{\tilde N}}}+\frac{1}{\tilde a \sqrt {n_{U}}}\right)\,.
$$
It can be easily seen that \cref{er2} degenerates into \cref{er1} for $\gamma = 1$.

\subsection{Variance Reduction}

It shows that our empirical risk estimators proposed is unbiased, and the excess risk can be bounded.
Similar to \citet{puauc}, the next question is that, in the incomplete scenario i.e., when $\gamma <1$, whether the variance of $\emprisk{\tilde P \tilde N U}(f)$ can be smaller than it of $\emprisk{\tilde P\tilde N}(f)$,
or whether $\mathcal{X}_U$ can help reduce the variance in estimating $\Rpn $.
To answer this question, pick any $f$ of interest. For simplicity, we assume that $n_U \to \infty$,
to illustrate the maximum variance reduction that could be achieved.

The variances and covariances are defined below:

$$
    \sigma^2_{\tilde P\tilde N}(f)=\operatorname{Var}_{\tilde P\tilde N}[\ell (f(x_{\tilde P}, x_{\tilde N}))]\,,
$$
$$
    \sigma^2_{\tilde P U}(f)=\operatorname{Var}_{\tilde P U}[\ell (f(x_{\tilde P}, x_U))]\,,
$$
$$
    \sigma^2_{U \tilde N}(f)=\operatorname{Var}_{U \tilde N}[\ell (f(x_U, x_{\tilde N}))]\,,
$$
$$
    \tau_{\tilde P\tilde N,\tilde P U}(f)=\operatorname{Cor}_{\tilde P\tilde N,\tilde P U}[\ell (f(x_{\tilde P}, x_{\tilde N})),\ell (f(x_{\tilde P}, x_U))]\,,
$$
$$
    \tau_{\tilde P\tilde N,U \tilde N}(f)=\operatorname{Cor}_{\tilde P\tilde N,U \tilde N}[\ell (f(x_{\tilde P}, x_{\tilde N})),\ell (f(x_{U}, x_{\tilde N}))]\,,
$$
$$
    \tau_{\tilde P U,U \tilde N}(f)=\operatorname{Cor}_{\tilde P U,U \tilde N}[\ell (f(x_{\tilde P}, x_U)),\ell (f(x_{U}, x_{\tilde N}))]\,.
$$

Then the following theorem can be obtained:

\begin{theorem}\label{vr}
    Assume $n_U \to \infty$, for any fixed $f$,
    the minimizers of the variance of the empirical risk $\emprisk{\tilde P \tilde N U}(f)$ is obtained by
    $$
    \gamma_{\tilde P \tilde N}=\arg\min_\gamma \operatorname{Var}[\emprisk{\tilde P \tilde N U}(f)]=\frac{\psi_{\tilde P \tilde N U}}{\psi_{\tilde P \tilde N U}-\psi _{\tilde P \tilde N}}\,,
    $$
    where
    \begin{align*}
    \psi_{\tilde P \tilde N}  &=\frac{1}{n_{\tilde P}n_{\tilde N}}\sigma^2_{\tilde P\tilde N}(f)\,,\text{ and}\\
    \psi_{\tilde P \tilde N U}&=\frac{1}{n_{\tilde P}}\tau_{\tilde P\tilde N,\tilde P U}(f)+\frac{1}{n_{\tilde N}}\tau_{\tilde P\tilde N,U \tilde N}(f)\,.
    \end{align*}
    In addition, we have $\operatorname{Var}[\emprisk{\tilde P \tilde N U}(f)]\leq \operatorname{Var}[\emprisk{\tilde P \tilde N}(f)]$ for any $\gamma \in (2\gamma_{\tilde P \tilde N}-1,1)$
    if $\psi _{\tilde P \tilde N U}>\psi _{\tilde P \tilde N}$.
\end{theorem}

\begin{proof}
    The empirical risk can be expressed as
    % $$
    % \begin{aligned}
    % &\emprisk{\tilde P \tilde N U}(f) \\
    % =&\gamma \emprisk{\tilde P \tilde N}(f) \!+\! (1\!-\!\gamma)(\emprisk{\tilde P U}(f)\!+\!\emprisk{U \tilde N}(f)\!-\!\frac{1}{2})\\
    % =&\frac{\gamma}{n_{\tilde P}n_{\tilde N}}\sum_{i=1}^{n_{\tilde P}}\sum_{j=1}^{n_{\tilde N}}\ell(f(x_i^{\tilde P}, x_j^{\tilde N})) \!+\! \frac{\gamma}{n_{\tilde P}n_{U}}\sum_{i=1}^{n_{\tilde P}}\sum_{j=1}^{n_{U}}\ell(f(x_i^{\tilde P}, x_j^{U}))\\
    % &\qquad \qquad \qquad \quad +\frac{\gamma}{n_{U}n_{\tilde N}}\sum_{i=1}^{n_{U}}\sum_{j=1}^{n_{\tilde N}}\ell(f(x_i^{U}, x_j^{\tilde N})) +\frac{1-\gamma}{2}\,.
    % \end{aligned}
    % $$
    $$
    \begin{aligned}
    \emprisk{\tilde P \tilde N U}(f)  ={}&\gamma \emprisk{\tilde P \tilde N}(f) \!+\! (1\!-\!\gamma)(\emprisk{\tilde P U}(f)\!+\!\emprisk{U \tilde N}(f)\!-\!\frac{1}{2})\\
    ={}& \frac{\gamma}{n_{\tilde P}n_{\tilde N}}\sum_{i=1}^{n_{\tilde P}}\sum_{j=1}^{n_{\tilde N}}\ell(f(x_i^{\tilde P}, x_j^{\tilde N})) \\
    &+ \frac{\gamma}{n_{\tilde P}n_{U}}\sum_{i=1}^{n_{\tilde P}}\sum_{j=1}^{n_{U}}\ell(f(x_i^{\tilde P}, x_j^{U}))\\
    &+\frac{\gamma}{n_{U}n_{\tilde N}}\sum_{i=1}^{n_{U}}\sum_{j=1}^{n_{\tilde N}}\ell(f(x_i^{U}, x_j^{\tilde N})) +\frac{1-\gamma}{2}\,.
    \end{aligned}
    $$
    Assuming $n_U\to \infty$, it follows that
    $$
    \begin{aligned}
        \operatorname{Var}[\emprisk{\tilde P \tilde N U}(f)]
        ={}&\gamma^2\frac{\sigma^2_{\tilde P \tilde N}}{n_{\tilde P}n_{\tilde N}}+2\gamma(1-\gamma)\frac{\tau_{\tilde P\tilde N,\tilde P U}(f)}{n_{\tilde{P}}}\\
         &\quad+2\gamma(1-\gamma)\frac{\tau_{\tilde P\tilde N,U \tilde N}(f)}{n_{\tilde{N}}}\\
        ={}&\gamma^2\psi_{\tilde P \tilde N} + 2\gamma(1-\gamma)\psi_{\tilde P \tilde N U}\,,
    \end{aligned}
    $$
    where the terms divided by $n_U$ are disappeared.

    To minimize the variance, the derivative is set to zero with respect to $\gamma$,
    $$
    \begin{aligned}
    \frac{\operatorname{Var}[\emprisk{\tilde P \tilde N U}(f)]}{\gamma}
    ={}& 2\gamma\psi_{\tilde P \tilde N}+(2-2\gamma)\psi_{\tilde P \tilde N U}\\
    {}={}& (2\psi_{\tilde P \tilde N}-2\psi_{\tilde P \tilde N U})\gamma + 2\psi_{\tilde P \tilde N U}\\
    {}={}& 0 \,.
    \end{aligned}
    $$
    Solving the equation then the minimizer of variance is obtained.
\end{proof}

\cref{vr} implies that the proposed risk estimator $\emprisk{\tilde P \tilde N U}$ 
have smaller variance than the standard supervised risk estimator $\emprisk{\tilde P\tilde N}$
if $\gamma$ is chosen properly, showing that the unlabeled data is helpful for building the model.

\section{rpAUC for Robust AUC Optimization}\label{sec:robust}

In previous sections, we developed a unified view of the weakly supervised AUC optimization problems from the perspective of AUC optimization with two contaminated sets.  Our findings provide a statistical consistent way to solve the AUC optimization problem in various weakly supervised scenarios.  
However, with finite training data, sorely training the models through above ERM problems still suffers from the contaminated supervision, the negative impact is reflected in the coefficients in the bounds (c.f. \(a\) in \cref{er1} and \(\tilde a\) in \cref{er2}). 

To mitigate this problem, in this section, we introduce a novel partial AUC, namely two-way reversed partial AUC (rpAUC), and demonstrate that by maximizing rpAUC through empirical risk minimization on contaminated data sets, WSAUC achieves robust AUC optimization.
We will first introduce the definition of rpAUC, which is obtained by reversing the restrictions of two-way pAUC{}. Then, we show the equivalency between the rpAUC optimization and the small loss trick, which is a commonly used technique for learning with noisy labels~\cite{Arpit2017dnnmem,Liu2016,Jiang2018mentornet,Ren2018}. This results in a straightforward solution to solve weakly supervised AUC optimization by simply replacing the full AUC risk with partial AUC risk.

Regarding the contaminated distribution \(p_A\) and \(p_B\) as positive and negative, the two-way reversed partial AUC (rpAUC) is defined as follows.
\begin{definition}
  Two-way Reverse Partial AUC with FPR threshold \(\alpha\) and TPR threshold \(\beta\) of model \(f\) can be defined as:
  \begin{align}
    \rpaucmo(f; \alpha, \beta) &= 1- \Ebb_{\xp\sim p_A^+(\xv)}[\Ebb_{\xn\sim p_B^-(\xv)}[\elloi(f(\xp, \xn))]]\,,\notag\\
    \text{where } p_A^+(\xv) &= p_A(\xv|f(\xv)\in[\tpr^{-1}_f(\beta), \infty)) \,,\notag\\
    p_B^-(\xv) &= p_B(\xv|f(\xv)\in(-\infty, \fpr^{-1}_f(\alpha)])\,. \notag
  \end{align}
\end{definition}

An illustration of rpAUC is shown in \cref{fig:rpauc}. By reversing the constraints of TPR and FPR, rpAUC trims the leftmost and the uppermost margin of the ROC curve. This is equivalent to eliminating \(\beta\) proportion of positive instances with bottom scores and \(\alpha\) proportion of negative instances with top scores by its definition. 
Next, to show that maximizing rpAUC is de facto removing instances include largest losses when maximizing AUC, we prove that the instances that induce largest losses are those who has bottom scores in \(\Xcala\) and top scores in \(\Xcalb\). 

\begin{algorithm}[t]
  \begin{algorithmic}
    \algrenewcommand\algorithmicindent{1em}%
    \State \textbf{Input:} Clean or noisy instance sets: \{\(\Xcalp[, \Xcalu][, \Xcaln]\)\}; 
    \State \textbf{Input:} hyper-parameters \(\alpha\), \(\beta\).
    \State Initialize model \(f\).
    \For{\(t = 1\rightarrow T\)}
      \For{\(k = 1\rightarrow K\)}
        \For{each pair of input sets \((\Xcala, \Xcalb)\)}
          \State Remove instances with top \(\alpha\) / \(\beta\) proportion losses.
          \State Sample batches \(B_A \in \Xcala^+, B_B \in \Xcalb^-\).
          \State Calculate rpAUC risk \(\empRab^{\mathrm{(rp)}}(f)\) on the batch.
          \State Update the model \(f\) by back propagation.
        \EndFor
      \EndFor
    \EndFor
    \State \textbf{Output:} The trained model \(f\).
  \end{algorithmic}
  \caption{rpAUC for Robust WSAUC Optimization}\label{alg:1}
\end{algorithm} 

The instance loss can be defined as follows:
\begin{equation*}
  L(x) = 
  \begin{cases}
  \frac{1}{|\Xcalb|}\sum_{\xn \in \Xcalb}\ell(f(\xp,\xn))\,,\ &\text{if }\xp \in \Xcala\,,\\
  \frac{1}{|\Xcala|}\sum_{\xn \in \Xcala}\ell(f(\xn,\xp))\,,\ &\text{if }\xp \in \Xcalb\,.
  \end{cases}
\end{equation*}
Suppose the surrogate loss \(l(z)\) is monotonically nonincreasing, it is easy to show that \(L(x)\) varies monotonically with \(f(x)\). Hence we have the following proposition.

\begin{proposition}  
  Suppose \(f\) is a scoring function. Then its rpAUC risk \(\hat{R}^{\mathrm{(rp)}}_{AB}(f;\alpha, \beta)\) equals to full AUC risk \(\empRab(f)\) with \(\alpha\) proportion of instances induce top losses \(L(x)\) from \(\Xcalb\) and \(\beta\) proportion of instances induce top losses from \(\Xcala\) removed.
\end{proposition}

This proposition bridges the minimization of empirical rpAUC risk with clean label selection, which has been shown effective for learning from contaminated labels.
To achieve robust AUC optimization, in training phase we minimize the following empirical rpAUC risk:
\begin{equation}\label{eq:erm:rpauc}
  \hat{R}^{\mathrm{(rp)}}_{AB}(f;\alpha, \beta) = \frac{1}{|\Xcala^+||\Xcalb^-|}\sum_{\xp \in \Xcala^+}\sum_{\xn \in \Xcalb^-}\ell(f(\xp,\xn))\,,
\end{equation}
where \(\Xcala^+\) is a set of instances in \(\Xcala\) with top \(\lfloor (1-\beta) |\Xcala| \rfloor\) scores, and \(\Xcalb^-\) is a set of instances in \(\Xcalb\) with bottom \(\lfloor (1-\alpha) |\Xcalb| \rfloor\) scores. 
This can be achieved by applying any existing two-way pAUC optimization  algorithm, e.g., \citet{Narasimhan2013opauc,Yang2021tpaucpie,auc-dro}, but reverses the instance selection in two sets. 
\cref{alg:1} shows a possible procedure for rpAUC optimization, which is simply based on pairwise gradient.

\begin{table*}[t]
  \caption{Performance (std) of noisy label AUC optimization. The best performances are bolded.}
  \label{tab:exp_noise}
  \small
  \centering
  \renewcommand{\arraystretch}{1.1}
  \renewcommand\tabularxcolumn[1]{m{#1}} % vertical centering X column
  \setcellformat[l]{00.0}{\ensuremath{_\pm}}{0.00}

  \begin{tabularx}{\textwidth}{
      @{\quad}
      % @{\extracolsep{4pt}}
      >{\hsize=.8\hsize}>{\centering\arraybackslash}X
      >{\hsize=.92\hsize}>{\centering\arraybackslash}X
      >{\hsize=.7\hsize}>{\centering\arraybackslash}X
      >{\hsize=.7\hsize}>{\centering\arraybackslash}X 
      >{\hsize=.7\hsize}>{\centering\arraybackslash}X 
      >{\hsize=.7\hsize}>{\centering\arraybackslash}X 
      >{\hsize=.7\hsize}>{\centering\arraybackslash}X 
      >{\hsize=.7\hsize}>{\centering\arraybackslash}X 
      >{\hsize=.7\hsize}>{\centering\arraybackslash}X 
      >{\hsize=.7\hsize}>{\centering\arraybackslash}X 
      >{\hsize=.7\hsize}>{\centering\arraybackslash}X 
      @{\quad}
      }
  \toprule
  \multirow{2}{*}[-.5ex]{Dataset}
  &Pos.\ noise   &           & 20\%      &           &           & 30\%      &           &           & 40\%      & \\
        \cmidrule(lr){3-5}                   \cmidrule(lr){6-8}                   \cmidrule(lr){9-11}
  &Neg.\ noise   & 20\%      & 30\%      & 40\%      & 20\%      & 30\%      & 40\%      & 20\%      & 30\%      & 40\%\\
  \midrule
\multirow{6}{*}{ MNIST }
& DRAUC       & \msc{87.2,4.3} & \msc{87.1,4.8} & \msc{87.1,4.4} & \msc{89.7,4.5} & \msc{88.9,4.9} & \msc{80.7,6.4} & \msc{88.4,4.9} & \msc{84.5,6.0} & \msc{74.1,8.8}           \\
& \auchinge   & \msb{99.6,0.1} & \msb{99.6,0.1} & \msc{99.4,0.1} & \msb{99.6,0.1} & \msb{99.5,0.1} & \msc{99.2,0.1} & \msc{99.4,0.1} & \msc{99.2,0.1} & \msc{98.7,0.3}           \\
& \aucramp    & \msc{99.3,0.1} & \msc{99.2,0.1} & \msc{99.0,0.1} & \msc{99.2,0.1} & \msc{99.0,0.1} & \msc{98.8,0.1} & \msc{99.0,0.1} & \msc{98.8,0.2} & \msc{98.2,0.3}           \\
& \aucunhinged& \msc{88.5,0.8} & \msc{88.5,0.8} & \msc{88.4,0.8} & \msc{88.4,0.8} & \msc{88.3,0.8} & \msc{88.2,0.7} & \msc{88.2,0.8} & \msc{88.1,0.7} & \msc{87.8,0.6}           \\
& \aucbarrier & \msc{99.5,0.1} & \msc{99.5,0.1} & \msb{99.5,0.1} & \msc{99.5,0.1} & \msb{99.5,0.1} & \msb{99.5,0.1} & \msb{99.6,0.0} & \msb{99.5,0.0} & \msb{99.5,0.1}  \\
& \rpauc      & \msb{99.6,0.0} & \msb{99.6,0.0} & \msb{99.5,0.1} & \msb{99.6,0.0} & \msb{99.5,0.0} & \msc{99.4,0.1} & \msc{99.5,0.1} & \msc{99.4,0.1} & \msc{99.1,0.1}           \\
\midrule
\multirow{6}{*}{ FMNIST }
& DRAUC       & \msc{93.5,4.8} & \msc{89.0,5.8} & \msc{78.7,5.7} & \msc{91.2,5.0} & \msc{89.5,7.0} & \msc{80.9,6.9} & \msc{90.1,6.9} & \msc{88.6,5.8} & \msc{76.4,9.0}           \\
& \auchinge   & \msc{99.3,0.1} & \msc{99.2,0.1} & \msc{99.1,0.1} & \msc{99.2,0.1} & \msc{99.2,0.1} & \msc{99.0,0.1} & \msc{99.1,0.1} & \msc{99.0,0.1} & \msc{98.8,0.1}           \\
& \aucramp    & \msc{99.2,0.1} & \msc{99.2,0.1} & \msc{99.1,0.1} & \msc{99.1,0.1} & \msc{99.1,0.1} & \msc{98.9,0.1} & \msc{99.0,0.1} & \msc{98.9,0.1} & \msc{98.6,0.2}           \\
& \aucunhinged& \msc{96.6,0.7} & \msc{96.5,0.7} & \msc{96.4,0.7} & \msc{96.7,0.7} & \msc{96.6,0.7} & \msc{96.6,0.7} & \msc{96.8,0.6} & \msc{96.7,0.7} & \msc{96.7,0.6}           \\
& \aucbarrier & \msc{98.9,0.1} & \msc{98.9,0.1} & \msc{98.8,0.1} & \msc{98.9,0.1} & \msc{98.9,0.1} & \msc{98.9,0.1} & \msc{98.8,0.1} & \msc{98.8,0.1} & \msc{98.9,0.1}           \\
& \rpauc      & \msb{99.4,0.1} & \msb{99.4,0.1} & \msb{99.3,0.1} & \msb{99.3,0.0} & \msb{99.3,0.1} & \msb{99.2,0.1} & \msb{99.3,0.0} & \msb{99.2,0.1} & \msb{99.1,0.1}  \\
\midrule
\multirow{6}{*}{ CIFAR10 }
& DRAUC       & \msc{38.8,9.3} & \msc{40.2,9.6} & \msc{37.7,12.0} & \msc{39.6,9.1} & \msc{35.2,13.0} & \msc{40.1,8.7} & \msc{40.3,9.6} & \msc{30.0,13.0} & \msc{40.1,12.0}          \\
& \auchinge   & \msc{92.8,0.1} & \msc{92.0,0.3} & \msc{91.7,0.4}  & \msc{92.1,0.3} & \msc{91.3,0.2}  & \msc{90.1,0.1} & \msc{91.5,0.4} & \msc{90.3,0.3}  & \msc{88.2,0.4}           \\
& \aucramp    & \msc{87.1,5.5} & \msc{88.5,3.1} & \msc{89.4,3.2}  & \msc{89.1,1.3} & \msc{89.2,2.5}  & \msc{88.6,2.4} & \msc{87.8,1.9} & \msc{89.2,1.4}  & \msc{85.0,2.2}           \\
& \aucunhinged& \msc{73.5,1.6} & \msc{73.3,3.1} & \msc{75.5,3.5}  & \msc{72.7,2.1} & \msc{74.2,5.1}  & \msc{73.1,4.7} & \msc{71.6,1.2} & \msc{74.4,4.7}  & \msc{71.6,1.1}           \\
& \aucbarrier & \msc{88.5,0.7} & \msc{88.4,1.4} & \msc{89.0,1.0}  & \msc{88.0,1.4} & \msc{88.7,0.9}  & \msc{88.4,0.9} & \msc{88.1,0.8} & \msc{88.4,1.4}  & \msc{88.2,0.8}           \\
& \rpauc      & \msb{93.6,0.2} & \msb{92.8,0.2} & \msb{92.3,0.2}  & \msb{92.9,0.2} & \msb{92.2,0.2}  & \msb{91.1,0.3} & \msb{92.2,0.3} & \msb{90.9,0.4}  & \msb{89.5,0.2}  \\
\midrule
\multirow{6}{*}{ CIFAR100 }
& DRAUC       & \msc{47.4,5.7} & \msc{39.8,8.8} & \msc{37.8,7.3} & \msc{38.1,10.0} & \msc{44.4,7.8} & \msc{41.5,6.8} & \msc{44.8,5.5} & \msc{41.7,7.1} & \msc{43.6,5.9}           \\
& \auchinge   & \msc{85.4,0.5} & \msc{83.5,1.0} & \msc{81.7,1.0} & \msc{84.0,0.8}  & \msc{82.0,0.8} & \msc{80.1,1.2} & \msc{82.7,0.6} & \msb{80.4,0.7} & \msc{76.4,0.6}           \\
& \aucramp    & \msc{86.3,1.4} & \msc{82.6,5.7} & \msc{80.7,3.6} & \msc{83.6,2.6}  & \msb{83.7,1.5} & \msc{77.0,5.6} & \msc{80.0,3.2} & \msc{74.6,5.6} & \msc{67.4,2.6}           \\
& \aucunhinged& \msc{63.0,2.4} & \msc{61.5,0.6} & \msc{62.7,1.5} & \msc{64.6,4.0}  & \msc{63.4,3.1} & \msc{63.5,2.0} & \msc{62.8,3.4} & \msc{64.9,2.4} & \msc{61.5,0.9}           \\
& \aucbarrier & \msc{78.1,2.0} & \msc{78.9,2.9} & \msc{78.5,2.7} & \msc{79.0,3.7}  & \msc{79.1,4.2} & \msc{78.8,3.1} & \msc{79.4,4.0} & \msc{78.2,4.3} & \msb{78.3,2.5}  \\
& \rpauc      & \msb{86.4,0.4} & \msb{85.3,0.5} & \msb{83.3,0.7} & \msb{84.8,0.5}  & \msc{82.5,0.7} & \msb{80.3,1.0} & \msb{83.1,0.5} & \msb{80.4,1.1} & \msc{76.5,1.3}           \\
 \bottomrule          
  \end{tabularx}
\end{table*}

\begin{table*}[tp]
  \caption{Performance (std) of positive-unlabeled AUC optimization. The best performances are bolded.}
  \label{tab:exp_pu}
  \small
  \centering
  \renewcommand{\arraystretch}{1.1}
  \begin{tabularx}{\textwidth}{
      @{\quad}
      >{\hsize=.9\hsize}>{\centering\arraybackslash}X
      >{\hsize=.99\hsize}>{\centering\arraybackslash}X
      @{\quad}
      >{\hsize=.9\hsize}>{\centering\arraybackslash}X
      >{\hsize=.9\hsize}>{\centering\arraybackslash}X
      >{\hsize=.9\hsize}>{\centering\arraybackslash}X
      >{\hsize=.9\hsize}>{\centering\arraybackslash}X
      >{\hsize=.9\hsize}>{\centering\arraybackslash}X
      >{\hsize=.9\hsize}>{\centering\arraybackslash}X
      % @{\quad} 
      }
  \toprule
  \multirow{2}{*}[-.5ex]{Dataset}
  &Pos.\ ratio   &\multicolumn{2}{c}{20\%} & \multicolumn{2}{c}{30\%} & \multicolumn{2}{c}{40\%} \\
        \cmidrule(r){3-4}         \cmidrule(r){5-6}          \cmidrule(r){7-8}
  &Label ratio   & 5\%       & 10\%       & 5\%        & 10\%       & 5\%         & 10\%       \\
  \midrule
\multirow{3}{*}{ MNIST }
& PU-AUC   & \msc{86.8,3.1} & \msc{96.8,0.5} & \msc{77.8,7.4} & \msc{94.8,1.1} & \msc{67.2,0.7} & \msc{88.8,2.9}           \\
& \samultpu& \msc{86.7,3.7} & \msc{97.5,0.4} & \msc{74.5,5.7} & \msc{95.3,1.1} & \msc{65.8,8.4} & \msc{89.1,2.9}           \\
& \rpauc   & \msb{93.8,3.2} & \msb{98.6,0.1} & \msb{88.5,3.5} & \msb{98.2,0.2} & \msb{84.2,4.6} & \msb{95.7,1.3}  \\
\midrule
\multirow{3}{*}{ FMNIST }
& PU-AUC   & \msc{90.7,2.4} & \msc{98.5,0.3} & \msc{81.1,6.1} & \msc{98.2,0.3} & \msc{67.6,10.4} & \msc{97.2,0.6}           \\
& \samultpu& \msc{92.4,2.1} & \msc{98.6,0.3} & \msc{83.0,6.3} & \msc{98.4,0.3} & \msc{67.9,10.6} & \msc{97.3,0.6}           \\
& \rpauc   & \msb{98.1,1.1} & \msb{99.0,0.1} & \msb{89.0,1.2} & \msb{98.9,0.1} & \msb{88.8,1.7}  & \msb{98.6,0.1}  \\
\midrule
\multirow{3}{*}{ CIFAR10 }
& PU-AUC   & \msc{62.6,7.1} & \msc{81.7,1.2} & \msc{57.6,6.7} & \msc{77.2,3.1} & \msc{51.8,6.3} & \msc{70.1,2.2}           \\
& \samultpu& \msc{56.9,7.1} & \msc{81.4,1.8} & \msc{58.9,7.7} & \msc{78.0,2.2} & \msc{51.9,6.6} & \msc{72.5,3.7}           \\
& \rpauc   & \msb{76.6,7.9} & \msb{86.7,0.7} & \msb{73.5,8.7} & \msb{82.7,1.9} & \msb{65.4,9.3} & \msb{76.9,5.1}  \\
\midrule
\multirow{3}{*}{ CIFAR100 }
& PU-AUC   & \msc{56.0,3.6} & \msc{68.9,0.9} & \msc{54.4,3.2} & \msc{65.8,2.1} & \msc{52.6,3.1} & \msc{60.9,2.8}           \\
& \samultpu& \msc{54.2,3.8} & \msc{69.6,1.3} & \msc{52.7,3.5} & \msc{67.1,1.0} & \msc{51.1,3.2} & \msc{62.4,1.8}           \\
& \rpauc   & \msb{64.4,5.4} & \msb{73.7,1.0} & \msb{62.4,5.5} & \msb{70.5,1.1} & \msb{59.6,5.6} & \msb{66.6,2.9}  \\
 \bottomrule          
  \end{tabularx}
\end{table*}

\section{Experiments}\label{sec:exp}

In this section, we empirically evaluate the proposed WSAUC for robust weakly supervised AUC optimization. We compare WSAUC with multiple baselines under different WSL scenarios to show that WSAUC provides a unified yet effective solution for building AUC-optimized models under these scenarios. Furthermore, we compare rpAUC against normal AUC as the training objective under varying noise ratios to demonstrate the benefits of using rpAUC in WSAUC when the labels are not perfectly clean.

\begin{table*}[tp]
  \caption{Performance (std) of multi-instance AUC optimization. The best performances are bolded.}
  \label{tab:exp_mil}
  \small
  \centering
  \renewcommand{\arraystretch}{1.1}
  \begin{tabularx}{.9\textwidth}{
      % @{\extracolsep{4pt}}
      >{\hsize=\hsize}>{\centering\arraybackslash}X
      >{\hsize=\hsize}>{\centering\arraybackslash}X
      >{\hsize=\hsize}>{\centering\arraybackslash}X
      >{\hsize=\hsize}>{\centering\arraybackslash}X
      >{\hsize=\hsize}>{\centering\arraybackslash}X
      >{\hsize=\hsize}>{\centering\arraybackslash}X }
  \toprule
  Dataset   & Musk1              & Musk2              & fox                & tiger              & elephant           \\
  \midrule
miSVM      & \msc{78.8,10.5} & \msc{75.7,12.2} & \msc{52.8,9.8}  & \msc{78.6,9.1}  & \msc{76.7,10.1}          \\
MISVM      & \msc{83.3,12.0} & \msc{83.9,14.6} & \msc{55.5,10.8} & \msc{83.6,10.3} & \msc{88.4,6.0}           \\
MissSVM    & \msc{78.9,10.1} & \msc{75.8,12.5} & \msc{50.1,9.4}  & \msc{78.2,8.9}  & \msc{77.4,9.2}           \\
SIL        & \msc{90.7,9.1}  & \msc{75.9,13.4} & \msc{59.3,11.2} & \msc{85.9,10.6} & \msc{86.1,6.4}           \\
sbMIL      & \msc{74.5,18.9} & \msc{73.9,16.0} & \msc{63.0,8.4}  & \msc{78.9,7.9}  & \msc{83.4,7.5}           \\
MIDAM (smx)& \msc{83.4,12.0} & \msc{90.5,6.8}  & \msc{62.2,18.8} & \msc{86.1,7.1}  & \msc{87.3,10.4}          \\
MIDAM (att)& \msc{82.6,10.7} & \msc{84.3,10.7} & \msc{73.3,9.7}  & \msc{86.7,6.6}  & \msc{90.6,6.9}           \\
\rpauc     & \msb{96.0,6.0}  & \msb{98.6,4.5}  & \msb{90.4,5.5}  & \msb{95.9,3.7}  & \msb{96.6,7.1}  \\
 \bottomrule          
  \end{tabularx}
\end{table*}

\begin{table*}[tp]
  \caption{Performance (std) of semi-supervised AUC optimization with/without noisy labels. The best performances are bolded.}
  \label{tab:exp_ssl}
  \small
  \centering
  \renewcommand{\arraystretch}{1.1}
  \begin{tabularx}{\textwidth}{
      @{\quad}
      >{\hsize=.9\hsize}>{\centering\arraybackslash}X
      >{\hsize=.9\hsize}>{\centering\arraybackslash}X
      @{\quad}
      >{\hsize=.9\hsize}>{\centering\arraybackslash}X
      >{\hsize=.9\hsize}>{\centering\arraybackslash}X
      >{\hsize=.9\hsize}>{\centering\arraybackslash}X
      >{\hsize=.9\hsize}>{\centering\arraybackslash}X
      >{\hsize=.9\hsize}>{\centering\arraybackslash}X
      >{\hsize=.9\hsize}>{\centering\arraybackslash}X 
      @{\quad}}
  \toprule
  \multirow{2}{*}[-.5ex]{Dataset}
  &Noise ratio   &\multicolumn{2}{c}{ 0\%} & \multicolumn{2}{c}{20\%} & \multicolumn{2}{c}{30\%} \\
        \cmidrule(r){3-4}         \cmidrule(r){5-6}          \cmidrule(r){7-8}
  &Label ratio   &  5\%       & 10\%       &  5\%       & 10\%        &  5\%       & 10\%         \\
  \midrule
\multirow{5}{*}{ MNIST }
& SAUC-LS & \msc{93.9,6.5} & \msc{97.5,6.6} & \msc{78.9,1.7} & \msc{87.3,4.7} & \msc{73.6,0.2} & \msc{73.7,0.1}           \\
& OptAG   & \msc{97.1,5.2} & \msc{98.4,4.4} & \msc{90.3,5.7} & \msc{96.4,8.0} & \msc{76.7,1.0} & \msc{80.6,2.3}           \\
& PNU-AUC & \msc{95.9,2.3} & \msc{98.5,0.1} & \msc{91.2,2.7} & \msc{98.5,0.2} & \msc{83.2,4.2} & \msc{97.5,0.5}           \\
& \samult & \msb{98.4,2.8} & \msb{99.3,0.1} & \msc{93.2,2.5} & \msb{98.9,0.1} & \msc{83.9,3.9} & \msc{98.2,0.2}           \\
& \rpauc  & \msc{98.3,2.7} & \msb{99.3,0.1} & \msb{96.6,2.4} & \msc{98.8,0.1} & \msb{95.2,2.0} & \msb{98.3,0.1}  \\
\midrule
\multirow{5}{*}{ FMNIST }
& SAUC-LS & \msc{98.3,4.7} & \msc{98.7,2.3} & \msc{54.0,2.9} & \msc{60.0,1.7} & \msc{47.0,3.2} & \msc{49.6,2.2}           \\
& OptAG   & \msb{98.7,2.2} & \msc{98.9,1.0} & \msc{95.8,6.6} & \msc{98.3,4.3} & \msc{52.7,3.2} & \msc{57.3,2.0}           \\
& PNU-AUC & \msc{97.8,0.8} & \msc{98.9,0.1} & \msc{93.5,1.5} & \msc{98.9,0.1} & \msc{89.8,2.9} & \msc{98.7,0.1}           \\
& \samult & \msc{98.4,0.9} & \msb{99.3,0.0} & \msc{96.1,1.1} & \msb{99.0,0.1} & \msc{90.3,2.6} & \msb{98.8,0.1}  \\
& \rpauc  & \msc{98.4,0.9} & \msb{99.3,0.1} & \msb{97.7,0.9} & \msb{99.0,0.1} & \msb{97.4,0.9} & \msb{98.8,0.1}  \\
\midrule
\multirow{5}{*}{ CIFAR10 }
& SAUC-LS & \msc{63.7,0.7} & \msc{63.9,0.6} & \msc{63.4,0.7} & \msc{63.0,0.5} & \msc{63.3,0.6} & \msc{62.5,0.3}           \\
& OptAG   & \msc{63.3,0.8} & \msc{58.2,2.3} & \msc{62.9,0.8} & \msc{58.4,2.0} & \msc{63.2,0.7} & \msc{58.0,1.7}           \\
& PNU-AUC & \msb{78.5,5.8} & \msc{88.9,0.1} & \msc{60.8,7.8} & \msc{85.4,1.1} & \msc{57.9,8.1} & \msc{82.5,0.9}           \\
& \samult & \msc{70.3,7.2} & \msc{89.2,0.4} & \msc{65.8,8.4} & \msb{86.9,0.8} & \msc{60.8,8.7} & \msb{84.5,0.6}  \\
& \rpauc  & \msc{70.9,7.1} & \msb{89.4,0.4} & \msb{66.5,8.0} & \msc{86.7,0.7} & \msb{65.8,0.7} & \msc{83.9,0.7}           \\
\midrule
\multirow{5}{*}{ CIFAR100 }
& SAUC-LS & \msc{42.5,1.2} & \msc{43.5,0.5} & \msc{41.8,1.1} & \msc{42.7,0.5} & \msc{41.8,1.0} & \msc{42.5,0.4}           \\
& OptAG   & \msc{42.2,1.0} & \msc{45.7,1.5} & \msc{41.6,0.8} & \msc{45.8,2.0} & \msc{41.7,1.0} & \msc{46.3,2.0}           \\
& PNU-AUC & \msc{61.0,4.9} & \msc{76.8,0.5} & \msc{56.6,4.1} & \msc{73.8,0.4} & \msc{55.4,4.3} & \msc{69.6,1.3}           \\
& \samult & \msb{64.0,4.5} & \msb{78.1,0.7} & \msc{60.1,4.7} & \msb{75.1,0.3} & \msc{57.4,4.7} & \msc{71.2,1.2}           \\
& \rpauc  & \msb{64.0,4.5} & \msc{77.9,0.6} & \msb{61.4,5.1} & \msb{75.1,0.5} & \msb{60.3,4.0} & \msb{71.7,1.0}  \\
 \bottomrule          
  \end{tabularx}
\end{table*}

\subsection{Experimental Setup}
The experiments under various weakly supervised learning scenarios are conducted with commonly used benchmark datasets. For AUC optimization with inaccurate or incomplete supervision, we utilize image benchmark \mnist{}, \fmnist{}, \cifar{}, and \cifarb{} to synthetic multiple datasets with varying setups per task. The datasets are transformed into binary classification datasets, i.e., odd vs.\ even numbers for \mnist{}, upper vs.\ lower garments for \fmnist{}, and animals vs.\ non-animals for CIFAR datasets. Such a dataset transformation is widely used in th related literature, e.g., \cite{feng2021,lu2019,Haim2022}. For AUC optimization with inexact supervision, we adopt several multi-instance learning datasets\footnote{http://www.uco.es/grupos/kdis/momil/} that are commonly used in the related literature, including Musk1, Musk2, fox, tiger, and elephant.

Our method is compared with multiple baselines designed for different problem settings, which will be introduced in the following subsections for each setting.
All methods are implemented with the same backbone and optimizer where applicable. Specifically, for \cifar{} and \cifarb{} datasets, the backbone is Mini-VGG; for \mnist{}, \fmnist{}, and the MIL datasets, the backbone is a simple fully-connected network with one hidden layer.
The hyper-parameters of the baselines are tuned through grid search or selected according to their original papers.
The hyper-parameter $\gamma$ of WSAUC is set as $0.45$ for all tasks.
All of the experiments are repeated for 10 times to eliminate the impact of the randomness, the AUC performance and the standard deviation are reported. For more details please refer to the released code.

\subsection{WSAUC in Various Weakly Supervised Settings}

\subsubsection{Noisy Label AUC Optimization}
For noisy label AUC optimization, we compare \rpauc{} with several noisy label learning losses/methods for AUC optimization:
\aucbarrier~\cite{Charoenphakdee2019}, a noisy AUC optimization method based on the barrier hinge loss, which is shown to be robust to the corrupted labels.
\auchinge, \aucramp, \aucunhinged{} are variations of \aucbarrier{} with different surrogate loss functions: hinge loss, ramp loss, and unhinged loss, respectively.
DRAUC~\cite{drauc}, a robust self-paced AUC optimization model against noisy and adversarial examples.
We compare the baselines on the task of AUC optimization with label noise, with varying noise ratios.
The noise ratios of the positive data and negative data are selected from \(\{20\%, 30\%, 40\%\}\), respectively, and all different combinations of the positive and negative ratios are tested for showing the performances under symmetrical and asymmetrical noise.

The results are shown in \cref{tab:exp_noise}. The results indicate that the performance of the multiple methods are close on the easy datasets \mnist{} and \fmnist{}, while \rpauc{} achieves relatively larger improvement on the challenging datasets \cifar{} and \cifarb{}. DRAUC shows unstable performance on CIFAR datasets, which may be due to the assumption of the generative models not holding on those datasets.

\subsubsection{Positive-Unlabeled AUC Optimization}

We compare the following baselines on the task of positive-unlabeled AUC optimization: PU-AUC~\cite{puauc} and \samultpu{}~\cite{Xie2018}.
PU-AUC minimizes an unbiased AUC risk by compensating the excess risk induced by the positive data in the unlabeled set. \samultpu{} is the degenerate version of \samult{}. It achieves PU AUC optimization by showing the consistency of risk obtained by treating the unlabeled data as negative.
The positive-unlabeled task is known to be sensitive to the class prior probability and label ratio. Thus, we test the situations where the positive ratios are chosen in \(\{20\%, 30\%, 40\%\}\), and the labeled ratios are chosen in \(\{5\%, 10\%\}\). We compare \rpauc{} with PU-AUC and \samultpu{}, which are the variations of PNU-AUC and \samult{}, respectively. 

The results are shown in \cref{tab:exp_pu}. It can be observed that \rpauc{} shows its advantage on most of the datasets and scenarios, especially when the label ratio is relatively low. This indicates that the advantage of \rpauc{} is greater when the amount of labeled data is small.

\subsubsection{Multi-Instance AUC Optimization}
For multi-instance AUC optimization, we compare \rpauc{} with multi-label learning approaches: MI-SVM and mi-SVM~\cite{mil}, two margin based approach for multi-instance learning. MissSVM~\cite{misssvm}, which uses a semi-supervised learning approach, treating the instances in positive bags as unlabeled data. SIL~\cite{sil}, a multi-instance learning approach that learns from the bag label. The multi-instance learning approach sbMIL~\cite{sbmil}, which is specially designed for the situation that the positive bags are sparse. And two varieties of MIDAM~\cite{midam}, which is a multi-instance AUC optimization method based on non-convex min-max optimization with smoothed-max (smx) and attention-based (att) pooling. Due to the challenge of synthesizing multi-instance tasks from image benchmark datasets, we conduct experiments on several widely used multi-instance learning datasets, i.e., Musk1, Musk2, fox, tiger, and elephant. The AUC is calculated based on the scores of the instance bags. For \rpauc{}, the model is trained at the instance level as described previously, and the bag score is predicted as the maximum instance score in the bag.

The results are shown in \cref{tab:exp_mil}. It is indicated that \rpauc{} achieves relatively large improvement, compared with miSVM, MISVM, MissSVM, SIL, and sbMIL. This may be because the these baselines are not explicitly optimizing AUC, indicating that explicitly optimizing instance-level AUC with \rpauc{} in the context of multi-instance learning is beneficial for achieve better bag-level AUC performance, and the robust training is helpful for the multi-instance learning task. Compared to MIDAM, a dedicated multi-instance method for AUC optimization, WSAUC still demonstrates significant performance improvement. This indicates the effectiveness of introducing rpAUC for multi-instance AUC optimization.

\subsubsection{Semi-Supervised AUC Optimization with Label Noise}
To the best of our knowledge, no existing AUC optimization method learns with both unlabeled and noisily labeled data simultaneously. Therefore, we compare the following baselines in the task of semi-supervised AUC optimization: SAUC-LS~\cite{Fujino2016}, a semi-supervised extension of AUC optimization with a novel loss function that utilizes the unlabeled data. OptAG~\cite{Fujino2016}, a semi-supervised AUC optimization method based on generative models. PNU-AUC~\cite{puauc}, a semi-supervised AUC optimization approach that combines unbiased PU and NU AUC optimization, where the unbiased PU and NU estimation is achieved by compensating the bias with the risk estimated on the labeled data. \samult{}~\cite{Xie2018}, a semi-supervised AUC optimization method that utilizes the unbiased risk estimation. In addition to showing the performance in the normal semi-supervised setup with 0\% noise ratio, we also evaluate the performance in situations where the labeled data is affected by noise, with noise ratios of 20\% or 30\%. We test the baselines with label ratios of \(\{5\%, 10\%\}\) to show the impact of the amount of labeled data.

The results are shown in \cref{tab:exp_ssl}. We observe that when the data is noisy-free, \rpauc{} achieves similar performance to PNU-AUC and \samult{} on all four datasets. However, when the labeled data is both limited (5\% label ratio) and noisy (20\% or 30\% noise ratio), \rpauc{} outperforms the baselines by a large margin. This indicates that \rpauc{} is particularly useful when dealing with scarce and contaminated labeled data. In other scenarios, \rpauc{} performs just as well as the prior state-of-the-art methods. Such results indicate that WSAUC is more robust to the label incomplete and inaccurate problems.

\subsection{rpAUC vs.\@ AUC as Training Objective}

To support our claim that rpAUC is a robust training objective for weakly supervised AUC optimization,
we empirically demonstrate how much the rpAUC can outperform AUC as the training objective when the noise ratio varies. Optimizing normal AUC can be regarded as a degenerate version of optimizing rpAUC by ablating the removal step of the noisy instance pairs. We follow the problem setup in \cref{subsec:unified} with varying \(\theta_A\) and \(\theta_B\) to show the performance increments of optimizing rpAUC over AUC{}.

We adopt the \cifarb{} dataset and vary the class proportions \(\theta_A\) and \((1-\theta_B)\) in the range of \([0.65, 1.0]\) with step size of \(0.05\). The smaller the values of \(\theta_A\) and \((1-\theta_B)\), the higher the noise ratio in the corresponding dataset. 5\% of the data is used for training, to investigate the algorithm performance when data is relatively scarce. The experiments are repeated 10 times for each class proportion combination. The results are shown in \cref{fig:vs}.

It can be observed in the plot that when the dataset is clean, maximizing AUC or rpAUC yields similar result. As the noise ratio increases, maximizing rpAUC shows greater improvement over maximizing AUC as the optimization objective during training.
This result validates that the rpAUC is a more robust training objective for weakly supervised AUC optimization.

\begin{figure}[t]
	\centering
  \includegraphics[width=.9\columnwidth]{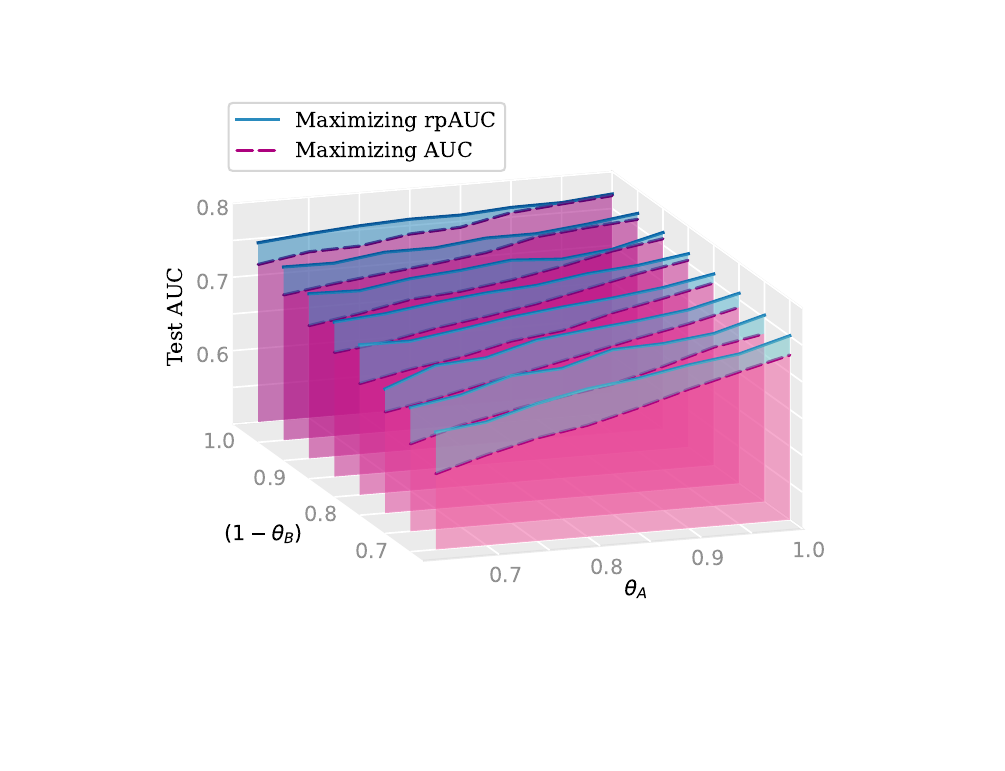} 
  \caption{Performance comparison of using AUC and rpAUC as training objective, under different noise ratios. Blue shading indicates the improvement in test AUC achieved by maximizing rpAUC during training.}
  \label{fig:vs}
\end{figure}

\subsection{Hyperparameter Sensitivity}

\begin{figure}[t]
	\centering
  \includegraphics[width=.99\columnwidth]{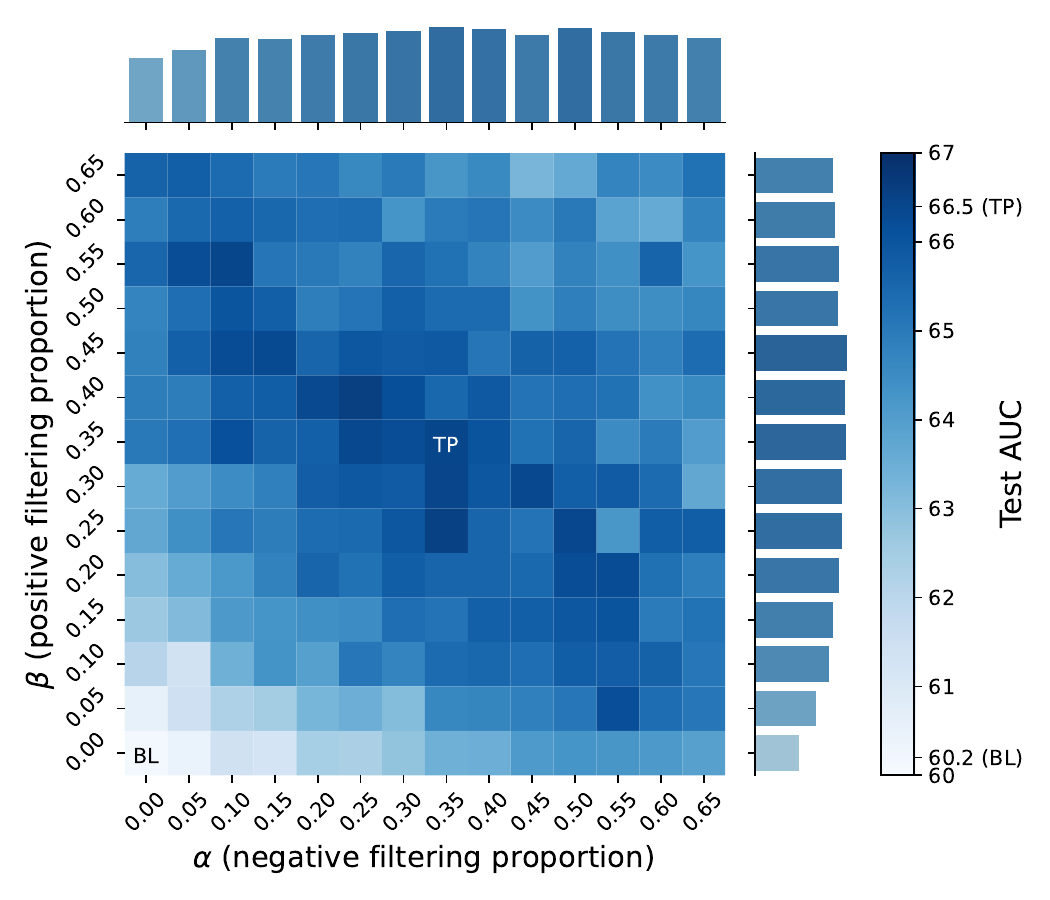} 
  \caption{Hyperparameter sensitivity. BL (baseline): rpAUC degenerates to full AUC when both \(\alpha\) and \(\beta\) are set to zero. TP (true proportion): the filtering proportions are identical to the true proportions, i.e., \(\alpha=1-\theta_A\) and \(\beta=\theta_B\).}
  \label{fig:ps}
\end{figure}

Using rpAUC requires two hyperparameters: the FPR threshold \(\alpha\) and the TPR threshold \(\beta\). These thresholds dictate the filtering proportions in the negative data (or \(\Xcalb\)) and the positive data (or \(\Xcala\)), respectively. To investigate how they impact the model performance, we alter their values and show the resulting variations in the model performance.

We again adopt the \cifarb{} dataset for the experiment, and the proportions \(\theta_A\) and \((1-\theta_B)\) are both set to \(0.65\). The data size is shrink to 5\% to enlarge the impact of the contamination. The hyperparameter \(\alpha\) and \(\beta\) are chosen in the range of \([0.65, 1.0]\) with step size of \(0.05\). The results are shown in \cref{fig:ps}.

In \cref{fig:ps}, the abbreviation ``BL'' denotes the baseline scenario, where both \(\alpha\) and \(\beta\) are set to zero. In this case, rpAUC degenerates to the full AUC\@.
The results reveal that as \(\alpha\) and \(\beta\) increase, the model performance consistently surpasses that of the baseline case, suggesting that a relatively arbitrary assignment of hyperparameter values can lead to performance enhancement. The test AUC reaches its peak when the hyperparameters closely approximate the true proportions  (annotated as ``TP''), aligning with our understanding of maximizing rpAUC\@. When the true proportions are unknown, it is advisable to make larger guesses of \(\alpha\) and \(\beta\). This is illustrated in the bar plot, where the average performance gradually declines as the hyperparameters exceed the true proportions.

\section{Conclusion and Future Work}\label{sec:con}

In this work, we study the weakly supervised AUC optimization problem. We present two main results: 
\begin{enumerate}
  \item We show that the AUC optimization problems in various weakly supervised scenarios can be uniformly framed as minimizing the AUC risk on contaminated sets, and the empirical risk minimization is consistent with the maximization of true AUC{}. 
  \item We propose the reversed partial AUC (rpAUC), which serves as a robust training objective for maximizing AUC under contaminated supervision. 
\end{enumerate}
Based on these results, we present the WSAUC framework, which provides a unified solution for multiple weakly supervised AUC optimization problems.

The unified WSAUC framework may be helpful for adapting specialize weakly supervised AUC optimization methods to new scenarios. It may also enable the development of new methods that can be applied to multiple scenarios, e.g., multi-class AUC optimization. 
These directions are left for future work.

\ifCLASSOPTIONcompsoc
  \section*{Acknowledgments}
\else
  \section*{Acknowledgment}
\fi

The authors would like to thank Prof.\ Lijun Zhang, Prof.\ Yu-Feng Li, and Prof.\ Wei Gao for their valuable advices, and the anonymous reviewers for their comments. This research was supported by NSFC (62076121, 61921006).

\ifCLASSOPTIONcaptionsoff
  \newpage
\fi



\bibliographystyle{IEEEtran}
\bibliography{IEEEabrv,uniauc,response}


% \end{thebibliography}


% \vfill 

\begin{IEEEbiography}[{\includegraphics[width=1in,height=1.25in,clip,keepaspectratio]{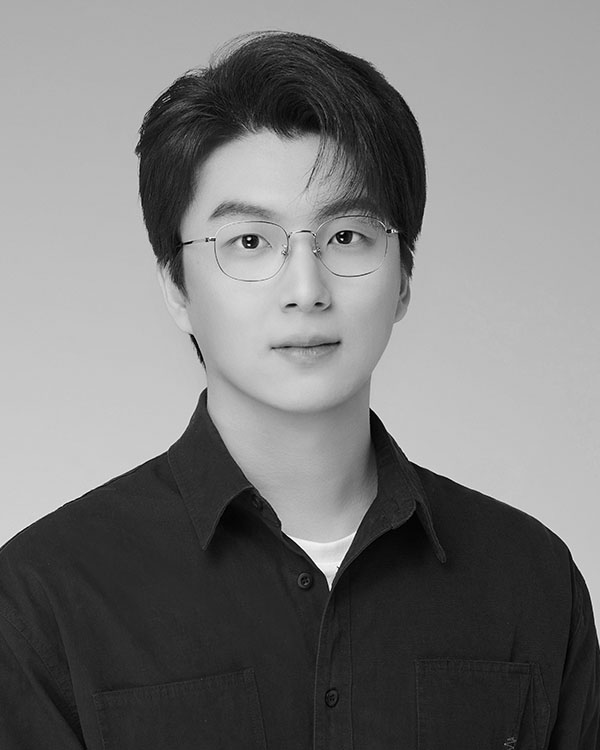}}]{Zheng Xie}
  received the BEng degree in computer science from Xi'an Jiaotong University, in 2016. Currently he is a PhD candidate in the Department of Computer Science and Technology, Nanjing University, and is a member of the LAMDA Group. His research interests mainly include machine learning and data mining, especially in learning from weak supervision or imbalanced supervision. He has received a number of awards including Outstanding Graduate of Xi'an Jiaotong University in 2016, AAAI Scholarship in 2018, selected in the Program for Outstanding PhD Candidate of Nanjing University in 2021, and so forth. 
  % He served as PC member or reviewer for several conferences and journals, including ICML, NeurIPS, AAAI, IJCAI, Pattern Recognition Letters, etc.
\end{IEEEbiography}

\begin{IEEEbiography}[{\includegraphics[width=1in,height=1.25in,clip,keepaspectratio]{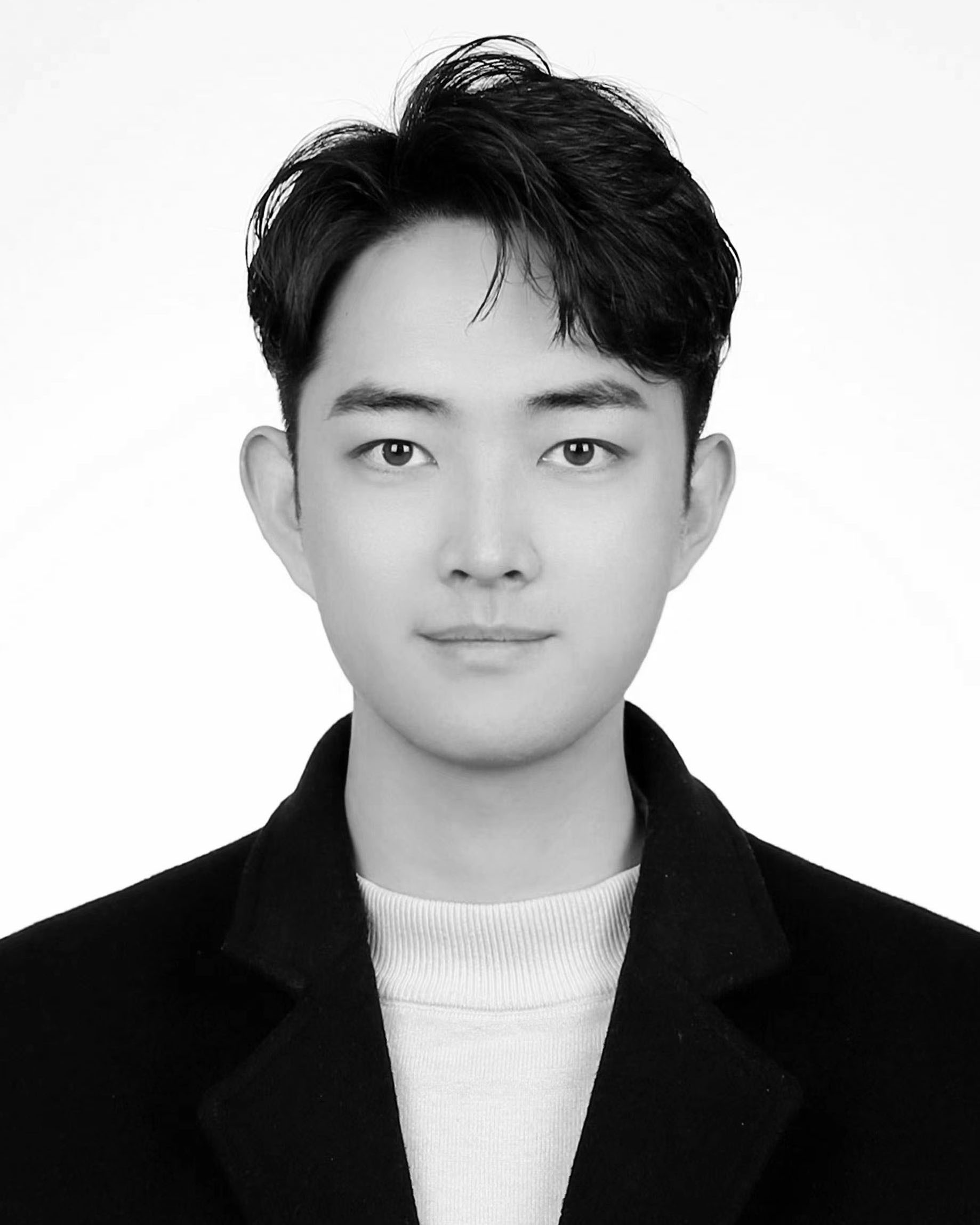}}]{Yu Liu}
  received the BSc degree in mathematics in Nanjing University, in 2021. Currently he is a master student in the School of Artificial Intelligence, Nanjing University, and is a member of the LAMDA Group. His research interests mainly include machine learning and data mining. 
\end{IEEEbiography}

\begin{IEEEbiography}[{\includegraphics[width=1in,height=1.25in,clip,keepaspectratio]{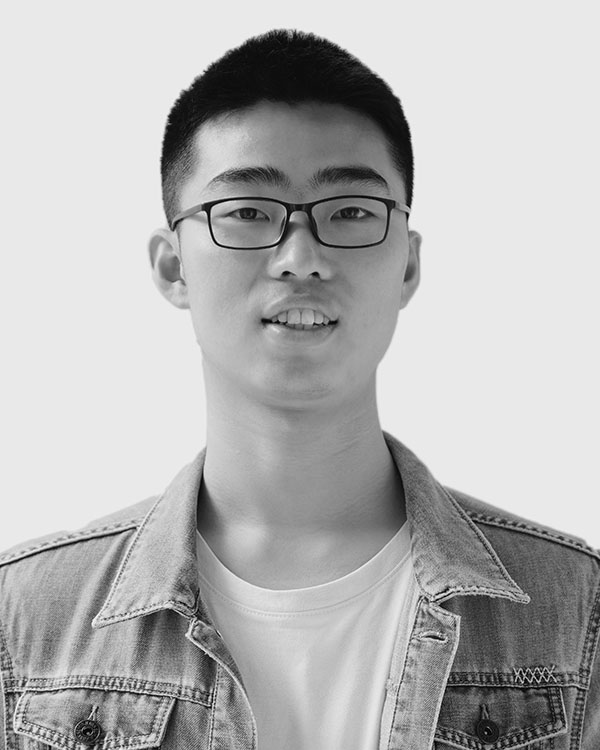}}]{Hao-Yuan He}
  received the BSc degree in computer science in Nanjing Tech University, in 2021. Currently, he is a master student in the School of Artificial Intelligence, Nanjing University, and is a member of the LAMDA Group. His research interests mainly include machine learning and data mining, especially in weakly supervised learning and neuro-symbolic learning. 
\end{IEEEbiography}

\begin{IEEEbiography}[{\includegraphics[width=1in,height=1.25in,clip,keepaspectratio]{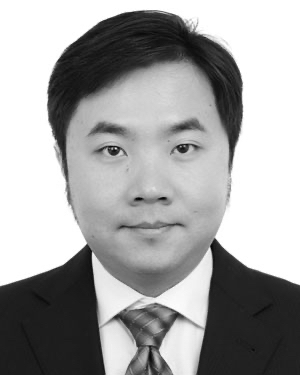}}]{Ming Li}
  (Member, IEEE) is currently a professor with the LAMDA group, the National Key Laboratory for Novel Software Technology, Nanjing University. His major research interests include machine learning and data mining, especially on software mining. He has served as the area chair of IJCAI, IEEE ICDM, etc, senior PC member of the premium conferences in artificial intelligence such as AAAI, and PC members for other premium conferences such as KDD, NeurIPS, ICML, etc. He is the founding chair of the International Workshop on Software Mining. He has been granted various awards including the PAKDD Early Career Award, the NSFC Excellent Youth Award, the New Century Excellent Talents program of the Education Ministry of China, etc.
\end{IEEEbiography}

\begin{IEEEbiography}[{\includegraphics[width=1in,height=1.25in,clip,keepaspectratio]{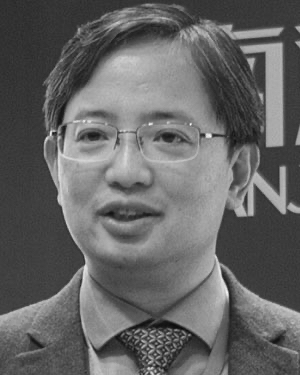}}]{Zhi-Hua Zhou}
(Fellow, IEEE) received the BSc, MSc, and PhD degrees in computer science from Nanjing University, China, in 1996, 1998, and 2000, respectively. He joined the Department of Computer Science \& Technology, Nanjing University as an assistant professor in 2001 and is currently a professor, the head of the Department of Computer Science and Technology, and the dean of the School of Artificial Intelligence. He is also the founding director of the LAMDA Group. He has authored the books Ensemble Methods: Foundations and Algorithms, Evolutionary Learning: Advances in Theories and Algorithms, Machine Learning (in Chinese) and authored or coauthored more than 200 papers in top-tier international journals or conference proceedings. He also holds 24 patents. His research interests include artificial intelligence, machine learning, and data mining. He was the recipient of various awards or honors including the National Natural Science Award of China, the IEEE Computer Society Edward J. McCluskey Technical Achievement Award, the ACML Distinguished Contribution Award, the PAKDD Distinguished Contribution Award, IEEE ICDM Outstanding Service Award, and the Microsoft Professorship Award. He is the editor-in-chief of the Frontiers of Computer Science, the associate editor-in-chief of the Science China Information Sciences, action or associate editor of Machine Learning, IEEE Transactions on Pattern Analysis and Machine Intelligence, and ACM Transactions on Knowledge Discovery from Data. He was the associate editor-in-chief for Chinese Science Bulletin, an associate editor for IEEE Transactions on Knowledge and Data Engineering, IEEE Transactions on Neural Networks and Learning Systems, ACM Transactions on Intelligent Systems and Technology, and Neural Networks. He founded Asian Conference on Machine Learning, was an advisory committee member for IJCAI from 2015 to 2016, steering committee member for ICDM, ACML, PAKDD and PRICAI, and the chair of various conferences, including the program chair of IJCAI 2021, AAAI 2019, the general chair of ICDM 2016, the senior area chair of NeurIPS, the area chair of ICML, AAAI, IJCAI, and KDD. He was the chair of the IEEE CIS Data Mining Technical Committee from 2015 to 2016, the chair of the CCF-AI from 2012 to 2019, and the chair of the CAAI Machine Learning Technical Committee from 2006 to 2015. He is a foreign member of the Academy of Europe, and a fellow of the ACM, AAAI, AAAS, IAPR, IET/IEE, CCF, and CAAI.
\end{IEEEbiography}

% \enlargethispage{-2in}


%\newpage



% \vfill

% \enlargethispage{-5in}


\end{document}